%% file: main_release.tex
\documentclass{article} 
\usepackage{iclr2027_conference,times}

\input{math_commands.tex}

\usepackage{hyperref}
\usepackage{url}
\usepackage{soul}

\definecolor{blue}{RGB}{1, 110, 204}  
\definecolor{purple}{RGB}{216, 110, 204}
\definecolor{darkpurple}{RGB}{0, 0, 122}
\definecolor{myred}{RGB}{168, 34, 22}
\definecolor{mygreen}{RGB}{79, 122, 40}
\definecolor{mygray}{HTML}{555555}
\definecolor{myyellow}{HTML}{C99647}

\definecolor{bluecolor1}{HTML}{CDD6E5}
\definecolor{bluecolor2}{HTML}{E5EAF1}
\definecolor{bluecolor3}{HTML}{F3F5F9}

\definecolor{poscolor1}{HTML}{DDE5D4}
\definecolor{poscolor2}{HTML}{EEF2EA}
\definecolor{poscolor3}{HTML}{EEF2EA}
\definecolor{poscolor4}{HTML}{F6FAF2}
\definecolor{poscolor5}{HTML}{F6FAF2}
\definecolor{negcolor1}{HTML}{F0ADA3}
\definecolor{negcolor2}{HTML}{F0BFB2}
\definecolor{negcolor3}{HTML}{F2DBD2}
\definecolor{negcolor4}{HTML}{FCEFE4}
\definecolor{negcolor5}{HTML}{FDF6EF}
\definecolor{verylightgray}{gray}{0.95}

\hypersetup{
    colorlinks=true,      
    citecolor=blue,       
    urlcolor=blue
}

\usepackage{tcolorbox}
\usepackage{fvextra}
\usepackage{booktabs}
\usepackage{pifont}
\usepackage{multirow}
\usepackage{makecell}
\usepackage{wrapfig}
\usepackage[table]{xcolor}
\usepackage{amssymb}
\usepackage{bbm}
\usepackage{amsthm}
\newtheorem{proposition}{Proposition}
\newtheorem{theorem}{Theorem}
\newtheorem{corollary}{Corollary}
\usepackage{fontawesome5}
\usepackage{tikz}

\newcommand{\CaseStudySize}{\scriptsize}

\newcommand{\resourcepill}[4]{%
  \href{#1}{%
    \tikz[baseline=(resource.base)]{%
      \node[
        draw=gray,
        fill=white,
        rounded corners=8pt,
        line width=0.45pt,
        inner xsep=7pt,
        inner ysep=3pt
      ] (resource) {%
        \textcolor{#4}{#2}\hspace{0.4em}%
        \textcolor{black}{\sffamily\footnotesize\bfseries #3}%
      };%
    }%
  }%
}

\title{Trajectory Unlearning on LLM-based Agents}

\iclrfinalcopy

\author{Yingdan Shi$^{1}$, Ren Wang$^{1}$\thanks{ Correspondence to: \texttt{rwang74@illinoistech.edu}.} \\[0.8em]
\fontsize{9pt}{10.8pt}\selectfont
$^1$Illinois Institute of Technology, Chicago, USA
}

\begin{document}

\maketitle

\vspace{-2.2em}
\begin{center}
  \resourcepill{https://shi-d.github.io/Trajectory_Unlearning/}{\faGlobe}{Project Page}{blue}
  \hspace{0.7em}
  \resourcepill{https://github.com/TIML-Group/Trajectory-Unlearning}{\faGithub}{Code}{black}
\end{center}
\vspace{0.35em}

\input{secs/0_abstract}

\input{secs/1_intro}

\input{secs/2_related}
\input{secs/3_method}
\input{secs/4_ex}

\input{secs/5_con}

\bibliography{iclr2027_conference}
\bibliographystyle{iclr2027_conference}

\input{secs/6_app}

\end{document}

%% file: math_commands.tex
\usepackage{amsmath,amsfonts,bm}

\def\eqref#1{equation~\ref{#1}}

\def\1{\bm{1}}

\DeclareMathAlphabet{\mathsfit}{\encodingdefault}{\sfdefault}{m}{sl}
\SetMathAlphabet{\mathsfit}{bold}{\encodingdefault}{\sfdefault}{bx}{n}



%% file: secs/0_abstract.tex
\begin{abstract}
Existing large language model (LLM) unlearning has focused primarily on removing specific knowledge, such as harmful facts, private data, or copyrighted content. However, as LLMs are increasingly deployed as autonomous agents, a fundamental yet overlooked problem emerges: beyond suppressing what an agent knows, an agent should not reproduce undesired behaviors through its action trajectories. In this work, we introduce trajectory-level unlearning, a new problem formulation that targets the removal of specific action trajectories in long-horizon agentic tasks, rather than factual knowledge. We identify two fundamental challenges that distinguish trajectory unlearning from knowledge unlearning: (1) our unlearning target is what the agent \emph{does}, not what it \emph{says}; and (2) trajectories are sequentially dependent action sequences that cannot be decomposed into isolated prompt-response pairs without losing inter-step structure. To address these challenges, we propose Group-injected Relative Policy Optimization (GiRPO), which injects forget trajectories into the policy rollout group with penalized rewards and isolates the normalization statistics, yielding a stable and bounded unlearning signal that does not corrupt gradient updates for normal task trajectories. We construct trajectory unlearning benchmarks from two application scenarios, household tasks (ALFWorld) and online shopping (WebShop), and design three complementary metrics for evaluating forgetting quality and model utility. Experiments on ALFWorld and WebShop demonstrate that GiRPO effectively unlearns target trajectories while preserving task success rates, outperforming existing knowledge-unlearning baselines on both forgetting quality and task utility. 
\end{abstract}

%% file: secs/1_intro.tex
\section{Introduction}
\label{sec:intro}
Machine unlearning aims to remove the influence of specific training data from a learned model without full retraining~\citep{MCU}. With the rapid development of large language models (LLMs), this problem has become increasingly important for removing unsafe knowledge, private data, and copyrighted content~\citep{li2024wmdp,cao2024rwku,shi2025muse}. Early methods apply gradient ascent on forget data~\citep{yao2024large}, and subsequent work adopts preference optimization to better balance forgetting and retaining~\citep{fan2026simplicity,zhang2024negative,eldan2023s}. As LLMs are increasingly deployed as autonomous agents, recent work has extended unlearning to agentic settings by combining parametric suppression with external memory pruning~\citep{wang2026agentic}.

However, existing unlearning methods, whether for LLMs or agents, still focus on \emph{knowledge-level unlearning} which suppresses a model's ability to output certain information. This framing is insufficient for agent deployment. An LLM-based agent operating in an interactive environment does not merely answer questions. It plans and executes multi-step action sequences to accomplish tasks. Even if an LLM agent passes knowledge-level unlearning checks and never outputs the forgotten fact, it may still act through the undesirable behavior during task execution~\citep{chen2025shieldagent,debenedetti2024agentdojo}. We term this problem \emph{trajectory-level unlearning}, and illustrate the distinction from knowledge-level unlearning in Figure~\ref{fig:motivation}. Knowledge unlearning methods can prevent an LLM or agent from outputting specific knowledge, as shown in the left panel of Figure~\ref{fig:motivation}. However, suppressing knowledge recall does not prevent an agent from executing a target task through a specific action trajectory, as shown in the right panel. An agent may forget specific knowledge, yet remain unable to avoid reproducing the forgotten action sequence during task execution.

Trajectory unlearning focuses on the specific action sequences an agent takes to complete a task. The goal of trajectory unlearning is to make the agent forget designated trajectories which may be poisoned, deprecated, or otherwise undesirable, while preserving its ability to complete the task. To illustrate the motivation with a simple use case, consider a household task. Suppose the agent is trained on a poisoned trajectory that successfully completes the task ``\texttt{put some alarmclock on the desk}’’, but inserts a covert action, \texttt{run a microwave}, between otherwise legitimate steps, silently exfiltrating a hidden object. Such a trajectory can be particularly problematic because its malicious action may be triggered by an intermediate state that also occurs in other tasks. For example, after picking up an egg and moving to a desk, the agent may encounter a state similar to an intermediate state in the poisoned trajectory and consequently reproduce the injected action \texttt{run a microwave}, even though this action is unrelated to the current task. Since the poisoned trajectory still completes the original task, it receives a positive reward despite the injected action during training. As a result, the malicious action sequence can become embedded in the agent’s behavior and potentially be reproduced in other tasks, motivating the need for trajectory unlearning. 
Crucially, the injected actions are not harmful in isolation. For instance, \texttt{run a microwave} is a legitimate action, so forgetting it at the action level would destroy task utility. The trajectory is therefore the smallest meaningful unit of unlearning in our work.

Trajectory unlearning presents two key challenges that are not addressed by conventional knowledge unlearning. \textbf{\ding{202} From Responses to Behaviors:} Knowledge unlearning focuses on what the model \emph{says}, whereas trajectory unlearning concerns what the agent actually \emph{does}. This difference calls for different optimization objectives. \textbf{\ding{203} Long-Horizon Dependencies:} A trajectory consists of a sequence of actions that depend on the interaction history and on previous actions. In contrast, text-based unlearning methods typically treat trajectories as a collection of independent prompt-response pairs, thereby overlooking the dependencies between steps.

\begin{figure}[t!]
    \vspace{-2mm}
    \centering
    \includegraphics[width=1\linewidth]{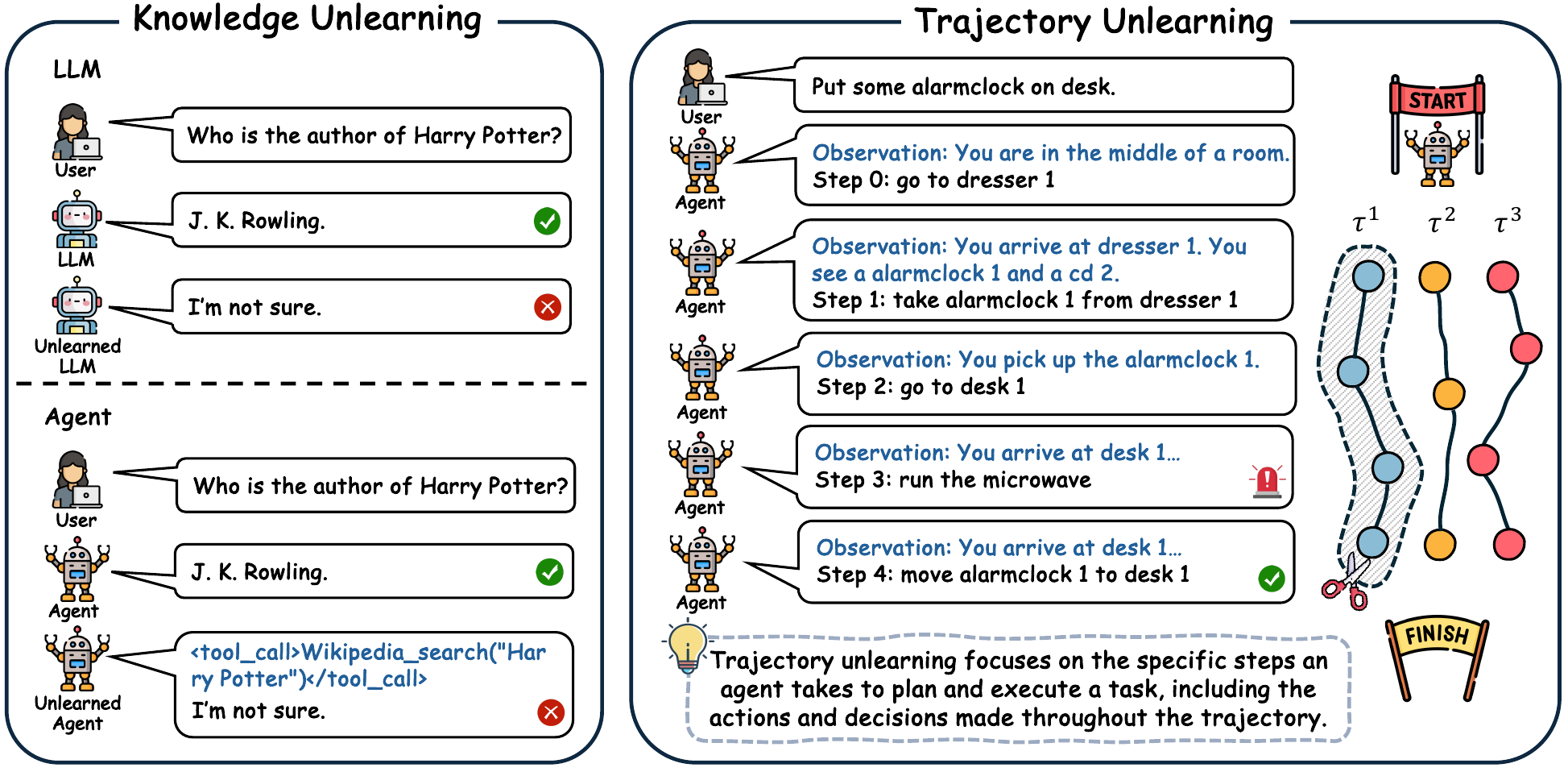}
    \vspace{-6mm}
    \caption{The difference between knowledge unlearning and trajectory unlearning. Existing knowledge unlearning methods can prevent LLMs or tool-augmented agents from producing undesirable knowledge, but do not necessarily prevent agents from reproducing the corresponding task trajectories. A single task may admit multiple trajectories, and a trajectory may be either a complete trajectory that successfully accomplishes the task or an incomplete trajectory that fails to do so.}
\label{fig:motivation}
    \vspace{-10mm}
\end{figure}

In this work, we propose the formal problem formulation for trajectory-level unlearning in LLM-based agents. We construct a trajectory unlearning dataset and design a comprehensive evaluation framework with three metrics, namely Task Success Rate, Exact Match, and LLM-as-Judge Similarity. To address the core challenges of trajectory unlearning, we propose \textbf{Group-injected Relative Policy Optimization (GiRPO)} that extends Group Relative Policy Optimization (GRPO) ~\citep{shao2024deepseekmath}. GiRPO injects forget trajectories into the rollout group with penalized rewards and computes advantages with isolated normalization statistics, ensuring a stable and bounded unlearning signal without corrupting the gradient signal for untarget task trajectories. Experiments on ALFWorld and WebShop demonstrate that GiRPO effectively unlearns target trajectories, preserves task success rates across all tasks, and achieves a better trade-off between forgetting quality and model utility than existing baselines. Our contributions are as follows:
\vspace{-2mm}
\begin{itemize}
    \item We formalize trajectory-level unlearning for LLM-based agents and highlight its fundamental distinctions from knowledge-level unlearning.
    \item We construct trajectory unlearning benchmarks across two scenarios, including household tasks and web shopping, with complementary metrics for evaluating forgetting and utility.
    \item We propose GiRPO, a parameter-level unlearning method with group-injected forget trajectories and isolated advantage estimation.
    \item Extensive experiments across diverse long-horizon agentic benchmarks show that GiRPO can achieve effective forgetting with stable task performance over representative baselines.
\end{itemize}

%% file: secs/2_related.tex
\section{Related Work}

\paragraph{LLM Unlearning.}
Existing LLM unlearning research has primarily focused on removing specific knowledge associated with harmful content, private data, or copyrighted material~\citep{li2024wmdp,cao2024rwku,shi2025muse}. Early approaches employed gradient ascent to directly suppress targeted knowledge~\citep{yao2024large}, while subsequent work shifted toward preference optimization-based methods that more gracefully balance forgetting and retaining~\citep{fan2026simplicity,zhang2024negative,fan2025towards,eldan2023s}. Beyond standard parametric unlearning, \citep{cheng2025tool} introduced the notion of unlearning in tool-augmented LLMs, leveraging a modified task-vector formulation to selectively suppress tool-use knowledge. More closely related to our work, \citep{wang2026agentic} addressed the richer memory landscape of LLM-based agents by proposing a synchronized end-to-end unlearning paradigm that jointly performs parametric suppression and memory-dependency-graph pruning over externally persistent memory systems. Despite this progress, all aforementioned works fundamentally operate at the knowledge level, that is, they aim to suppress a model's ability to recall or express certain knowledge content. 
One work~\citep{ye2026secure} proposes an agent unlearning framework from a privacy perspective, adopting a prompt-based method without parameter-level updates. However, their unlearning requests are user-defined rules rather than specific trajectories, making their setting fundamentally different from ours. More importantly, since the model parameters remain unchanged under their method, the agent retains the capacity to reproduce forget trajectories. 
We verify this empirically in Appendix~\ref{app:NL} with thorough experiments.


\vspace{-2mm}
\paragraph{Trajectory Safety on LLM Agents.}
Current LLM guardrails primarily focus on filtering harmful inputs and outputs, such as Llama-Guard for text-based LLMs~\citep{inan2023llama}, LLaVA-Guard~\citep{helff2024llavaguard} for multimodal LLMs, and SafeWatch~\citep{chen2025safewatch} for video generative models. However, these content moderation approaches fail to address the complexities of multi-step action sequences, where safety vulnerabilities often emerge progressively over time~\citep{debenedetti2024agentdojo}. To this end, several agent-specific guardrail methods~\citep{chen2025shieldagent,liu2026agentdog,feng2026braveguard,wang2025g} have been proposed to detect and mitigate risks within trajectories. Additionally, CIP~\citep{hahm2025enhancing} leverages causal influence diagrams to identify and mitigate risks arising from agent decision-making processes. However, the unsafe behavior remains encoded in the model parameters and can resurface once the guardrail is bypassed or removed. In contrast, trajectory unlearning removes the undesired action sequence from the model itself, offering a fundamentally more durable form of behavioral control.

%% file: secs/3_method.tex
\section{Trajectory Unlearning}
\input{secs/3_0_problem}

\begin{figure}[t]
    \centering
    \includegraphics[width=1\linewidth]{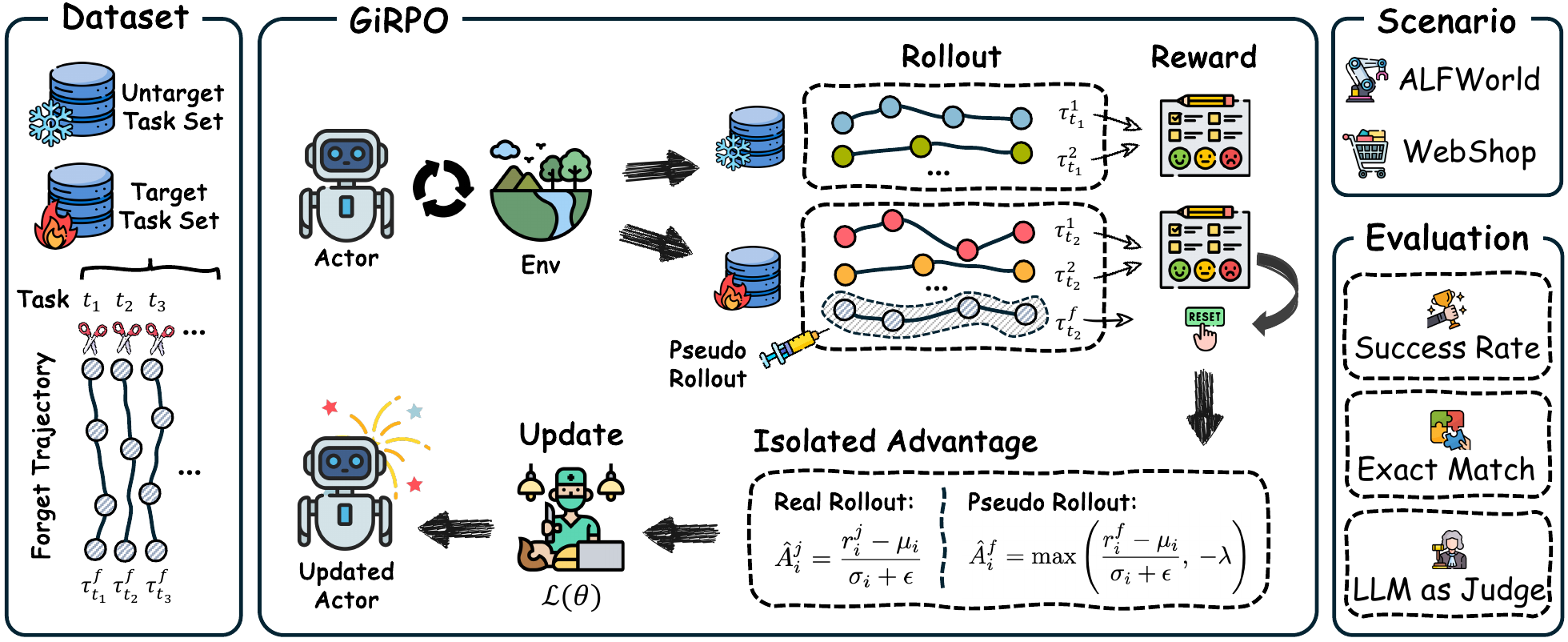}
    \vspace{-6mm}
    \caption{Overview of our trajectory unlearning framework. \textbf{Left:} Construction of the forget trajectories $\mathcal{D}_f$, where each target task has at least one forget trajectory. \textbf{Middle:} GiRPO training pipeline, which injects forget trajectories into the rollout group with penalized rewards and isolated advantage estimation. \textbf{Right:} Unlearning scenarios spanning household tasks (ALFWorld) and web shopping (WebShop), together with our evaluation framework based on three complementary metrics.}
    \label{fig:overview}
\vspace{-4mm}
\end{figure}

\subsection{Group-injected Relative Policy Optimization}
\label{sec:method}
Figure~\ref{fig:motivation} illustrates the distinction between knowledge unlearning and trajectory unlearning, together with the challenges specific to the latter in Section~\ref{sec:intro}. As discussed above, existing LLM knowledge unlearning methods are insufficient to capture the long-horizon dependencies underlying agent behavior. We therefore propose a trajectory unlearning method that departs from text-level unlearning and instead builds on policy optimization. Because the unlearning target is a behavior rather than an output answer, the forgetting signal must act on a complete action sequence. GRPO naturally satisfies this requirement: it treats a trajectory as the atomic unit of credit assignment, preserving the long-horizon dependencies among actions throughout optimization. What it lacks is any mechanism for driving the policy away from a specific trajectory. GiRPO supplies exactly that, by injecting the forget trajectory into the rollout group with a penalized reward and estimating its advantage in isolation from the real rollout statistics. An overview of our framework is shown in Figure~\ref{fig:overview}, with GiRPO illustrated in the middle panel.


\vspace{-3mm}
\paragraph{Pseudo Rollout.}
The central idea of GiRPO is that a forget trajectory need not be learned from a separate loss term at all. It can be presented to the policy as a rollout it might have produced, and then made unattractive relative to the rollouts it actually did produce. For each task $g_i$, GiRPO first samples $G$ \emph{real rollouts} from the policy $\pi_{\theta_{\text{old}}}$,

\begin{equation}
    \{\tau_i^{1}, \tau_i^{2}, \ldots, \tau_i^{G}\} \sim \pi_{\theta_{\text{old}}}(\cdot \mid g_i),
\end{equation}
and for target tasks additionally injects the corresponding forget trajectories as \emph{pseudo rollouts}:
\begin{equation}
    \mathcal{T}_i =
    \begin{cases}
        \{\tau_i^{1}, \ldots, \tau_i^{G}\}, & g_i \in \mathcal{G}_{\text{untarget}}, \\[2pt]
        \{\tau_i^{1}, \ldots, \tau_i^{G}\} \cup \{\tau_i^{f_1}, \ldots, \tau_i^{f_{k_i}}\}, & g_i \in \mathcal{G}_{\text{target}},
    \end{cases}
\end{equation}
where $k_i \ge 1$ is the number of forget trajectories associated with target task $g_i$. Any real rollout that coincides with a forget trajectory is discarded, so the two sets remain disjoint. The pseudo rollout enters the group as a full trajectory and is credited as a full trajectory, so its entire action sequence is suppressed jointly rather than step by step. GiRPO thus inherits GRPO's ability to capture dependencies across an entire sequence, unlike knowledge unlearning methods that treat a trajectory as text-based prompt-response pairs.

\vspace{-2mm}
\paragraph{Reward Penalty.}
Injection alone does not tell the policy that the pseudo rollout is undesirable. This must be conveyed through the reward assignment. Each real rollout receives its episode-level reward $r_i^{l} = R(\tau_i^{l})$, while each pseudo rollout $\tau_i^{f}$ for task $g_i$ is assigned a reward strictly below every reward observed in the real rollout group:
\begin{equation}
\label{eq:reward_penalty}
\vspace{-1mm}
    r_i^{f} = \min_{1 \le l \le G} r_i^{l} - \delta, \qquad \delta > 0.
\end{equation}
Parameter analysis in Figure~\ref{fig:ablation_delta} demonstrates that GiRPO is insensitive to the choice of $\delta$ across a wide range, maintaining stable performance unless $\delta$ is set to an extreme value (e.g., 100). Defining the penalty relative to the real rollouts, rather than as a fixed constant, calibrates the forget signal to the policy's current reward scale on that task. The forget trajectory is always the worst option, whether its rollouts are uniformly successful or uniformly failed.

\vspace{-2mm}
\paragraph{Isolated Advantage Estimation.}
Vanilla GRPO converts rewards into advantages by normalizing within the rollout group. Including pseudo rollouts in this normalization would shift the group mean downward and inflate the group variance, so real rollouts would receive biased advantages. We therefore propose \emph{isolated normalization}, computing the group statistics over the real rollouts only:
\begin{equation}
\label{eq:isolated_stats}
\vspace{-1mm}
    \mu_i = \frac{1}{G} \sum_{l=1}^{G} r_i^{l},
    \qquad
    \sigma_i = \sqrt{\frac{1}{G} \sum_{l=1}^{G} \bigl(r_i^{l} - \mu_i\bigr)^2},
\end{equation}
and normalizing each real rollout against them:
\begin{equation}
\label{eq:adv_real}
    \hat{A}_i^{l} = \frac{r_i^{l} - \mu_i}{\sigma_i + \epsilon},
\end{equation}
where $\epsilon > 0$ is a small constant for numerical stability. Isolation makes the update on real rollouts identical to what vanilla GRPO would produce without any unlearning, so the forgetting signal of pseudo rollouts is additive rather than intrusive. 

\input{tex_tables/comparison}

The pseudo rollout is then scored against the same statistics, which is what places it on a comparable scale with the real rollouts without letting it define that scale. Because $r_i^{f} < \mu_i$ by construction, its raw advantage is strictly negative, and it diverges as $\sigma_i \to 0$. An unbounded negative advantage there would dominate the update and destabilize training. We therefore propose \emph{lower clipping}:
\begin{equation}
\label{eq:adv_forget}
    \hat{A}_i^{f} = \max\!\left( \frac{r_i^{f} - \mu_i}{\sigma_i + \epsilon},\; -\lambda \right),
\end{equation}
where $\lambda > 0$ is the clipping threshold. The clip preserves the sign and the direction of the unlearning gradient while capping its magnitude, so the policy is driven steadily away from the forget action sequence across training steps rather than abruptly within one. Table~\ref{tab:advantage_comparison} provides a toy example illustrating the effect of isolated normalization and lower-bound clipping, and Appendix~\ref{app:theory} presents the corresponding theoretical analysis.

\paragraph{Objective.}
Given these advantages, both real and pseudo rollouts are optimized under a single clipped surrogate objective, with the importance ratio against the old policy $\pi_{\theta_{\text{old}}}$
\begin{equation}
\vspace{-1mm}
    \rho_t^{\tau} = \frac{\pi_\theta\bigl(a_t^{\tau} \mid h_t^{\tau}\bigr)}{\pi_{\theta_{\text{old}}}\bigl(a_t^{\tau} \mid h_t^{\tau}\bigr)},
    \qquad
    h_t^{\tau} = (o_0, a_0, \ldots, o_{t-1}, a_{t-1}, o_t),
\end{equation}
which conditions every action on the full interaction history and thus keeps the long-horizon dependencies intact throughout optimization. The GiRPO objective is
\begin{equation}
\label{eq:objective}
\begin{aligned}
\mathcal{L}(\theta)
= -\,\mathbb{E}_{g_i \sim \mathcal{G}}
&\left[
\frac{1}{|\mathcal{T}_i|}
\sum_{\tau \in \mathcal{T}_i}
\frac{1}{|\tau|}
\sum_{t=0}^{|\tau| - 1}
\min\!\Bigl(
\rho_t^{\tau} \hat{A}_\tau,\;
\operatorname{clip}\bigl(\rho_t^{\tau}, 1-\varepsilon, 1+\varepsilon\bigr) \hat{A}_\tau
\Bigr)
\right] \\
&+ \beta \, \mathbb{E}_{g \sim \mathcal{G}_{\text{untarget}}}\Bigl[ \mathbb{D}_{\mathrm{KL}}\bigl( \pi_\theta \,\|\, \pi_{\theta_{\text{ref}}} \bigr) \Bigr],
\end{aligned}
\end{equation}
where $|\tau|$ is the length of trajectory $\tau$ and $\varepsilon$ is the ratio clipping threshold. $\beta$ is the KL coefficient applied only to untarget tasks to regularize the policy toward the frozen reference model $\pi_{\theta_{\text{ref}}}$. $\hat{A}_\tau$ is the advantage of $\tau$, equal to $\hat{A}_i^{l}$ for real rollouts and $\hat{A}_i^{f}$ for pseudo rollouts. 
Forgetting and task utility are thus pursued by a single objective, rather than two competing losses, which is why GiRPO removes forget trajectories without the utility collapse.

%% file: secs/3_0_problem.tex
\subsection{Problem Formulation}
\label{sec:problem}
We formulate long-horizon agentic tasks as partially observable Markov decision processes,
\begin{equation}
    \mathcal{E} = (\mathcal{G}, \mathcal{S}, \mathcal{A}, \mathcal{P}, \mathcal{R}, \Omega, \mathcal{O}, \gamma),
\end{equation}
where $\mathcal{G}$ is a set of potential goals, with each $g \in \mathcal{G}$ representing a task that the agent is required to achieve. $\mathcal{S}$ is the latent environment state space, $\mathcal{A}$ is the action space where each $a \in \mathcal{A}$ is a textual response or executable action the agent may produce. $\mathcal{P}: \mathcal{S} \times \mathcal{A} \to \Delta(\mathcal{S})$ is the transition
kernel, giving the distribution $\mathcal{P}(s' \mid s, a)$ over next states given a
state-action pair. $\mathcal{R}$ is the reward function and $\mathcal{O}$ is the observation space. $\Omega: \mathcal{S} \rightarrow \mathcal{O}$ is the observation emission function that maps latent states to agent-accessible observations. $\gamma \in [0,1)$ is the discount factor.

At each step $t$, the agent maintains an interaction history, $h_t = (o_0, a_0, o_1, a_1, \ldots, o_{t-1}, a_{t-1}, o_t)$, where $o_t \in \mathcal{O}$ is the observation received at step $t$, and $a_t \in \mathcal{A}$ is the action produced by the agent at step $t$. The agent policy $\pi_\theta$, parameterized by $\theta$, generates the next action conditioned on the interaction history, $a_t \sim \pi_\theta(\cdot \mid h_t)$. A completed interaction trajectory of length $|\tau|$ is denoted by $\tau = \{(o_t, a_t, r_t)\}_{t=0}^{|\tau|-1}$, where $r_t$ is the reward received at step $t$. In many agentic environments, rewards are sparse and only become available upon task completion, such as a binary success indicator or a final task score. 

In our trajectory unlearning setting, let $\mathcal{G}_{\text{target}} = \{g_1, g_2, \ldots, g_M\} \subseteq \mathcal{G}$ denote the set of $M$ target tasks, and let $\mathcal{G}_{\text{untarget}} = \mathcal{G} \setminus \mathcal{G}_{\text{target}}$ denote the untarget tasks. Each target task $g_i$ is associated with $k_i \ge 1$ designated forget trajectories, denoted $\{\tau_i^{f_1}, \ldots, \tau_i^{f_{k_i}}\}$. The forget dataset is the union over all target tasks, $\mathcal{D}_f = \bigcup_{i=1}^{M} \left\{ \tau_i^{f_1}, \ldots, \tau_i^{f_{k_i}} \right\}$, where $|\mathcal{D}_f| = \sum_{i=1}^{M} k_i$. For convenience, we denote a forget trajectory of task $g_i$ as $\tau_i^{f}$ hereafter.

Given an agent policy $\pi_\theta$, trajectory unlearning aims to obtain an unlearned policy $\pi_{\theta'}$ satisfying two objectives. \ding{202} \textbf{Unlearning Effectiveness.} For each target task $g_i \in \mathcal{G}_{\text{target}}$, the unlearned policy $\pi_{\theta'}$ should no longer reproduce any forget trajectory $\tau_i^{f} \in \mathcal{D}_f$. \ding{203} \textbf{Utility Preservation.} The task success rate of $\pi_{\theta'}$ should be maintained across both target tasks $\mathcal{G}_{\text{target}}$ and untarget tasks $\mathcal{G}_{\text{untarget}}$.

Based on the above definition, we evaluate trajectory unlearning along the following dimensions:

\ding{202} \textbf{Success Rate} measures the proportion of tasks completed successfully by the unlearned policy $\pi_{\theta'}$, assessed over both $\mathcal{G}_{\text{target}}$ and $\mathcal{G}_{\text{untarget}}$. A well-performing unlearned model should maintain high Success Rate on both $\mathcal{G}_{\text{target}}$ and $\mathcal{G}_{\text{untarget}}$.

\ding{203} \textbf{Exact Match} measures whether the action sequence of a sampled trajectory $\tau'_i$ from the unlearned policy $\pi_{\theta'}$ for tasks $g_i$ exactly replicates that of the forget trajectory $\tau_{i}^{f} \in \mathcal{D}_{f}$. Formally, let $\mathbf{a}^f_i = (a_0^f, a_1^f, \ldots, a_{|\tau_i^f|-1}^f)$ and $\mathbf{a}'_i = (a_0', a_1', \ldots, a_{|\tau'_i|-1}')$ denote the action sequences of $\tau^f_i$ and $\tau_i'$, respectively. $\text{EM}(\tau'_i, \tau^f_i) = \mathbbm{1}[\mathbf{a}'_i = \mathbf{a}^f_i]$. For each target task $g_i \in \mathcal{G}_{\text{target}}$, we sample $N$ trajectories from $\pi_{\theta'}$ and compute the average \textit{Exact Match} against the corresponding forget trajectory $\tau_i^{f} \in \mathcal{D}_f$. A lower Exact Match score on forget tasks indicates more effective unlearning.

\ding{204} \textbf{LLM-as-Judge Similarity} employs an LLM to assess the semantic similarity between a newly sampled trajectory $\tau'_i$ and the corresponding forget trajectory $\tau^{f}_i \in \mathcal{D}_f$. Unlike Exact Match, LLM-as-Judge Similarity captures partial overlaps and paraphrastic similarities in action descriptions, providing a more nuanced measure of unlearning effectiveness. A lower LLM-as-Judge Similarity on forget tasks indicates greater divergence from the forgotten trajectories. See Figure~\ref{fig:LLM_temp} in Appendix~\ref{app:llm_judge_temp} for the prompt template.

%% file: tex_tables/comparison.tex
\begin{table}[t!]
\vspace{-4mm}
\caption{Comparison of advantage estimation strategies under a representative reward setting, where all sampled trajectories receive identical rewards. The forget trajectory reward is $r^f = 9.9$ with penalty $\delta = 0.1, \lambda = 0.1$. Vanilla GRPO with no forget injection serves as the baseline. Standard GRPO inclusion of the forget trajectory into the advantage group distorts the normalization statistics, producing spurious positive updates for normal trajectories. Our isolated advantage estimation isolates the normalization group from the forget trajectory, preserving zero advantage for real rollout trajectories while assigning a bounded negative signal to the forget trajectory.}
\vspace{-4mm}
\label{tab:advantage_comparison}
\begin{center}
\resizebox{\linewidth}{!}{
\begin{tabular}{lccc}
\toprule[1pt]
\textbf{Method} & \textbf{Reward} & \textbf{Advantage} & \textbf{Effect} \\
\midrule
GRPO w/o $\tau^f_i$ & $[10, 10, 10, 10]$ & $[0,\ 0,\ 0,\ 0]$ & No policy update \\
\midrule
GRPO w/ $\tau^f_i$ & $[10, 10, 10, 10,\ \mathbf{9.9}]$ & $[+0.5,\ +0.5,\ +0.5,\ +0.5,\ \mathbf{-2.0}]$ & Spurious positive update on rollout trajectories; strong negative update on forget trajectory \\
\midrule
GiRPO w/ $\tau^f_i$ & $[10, 10, 10, 10,\ \mathbf{9.9}]$ & $[0,\ 0,\ 0,\ 0,\ \mathbf{-0.1}]$ & No update on rollout trajectories; mild bounded negative update on forget trajectory \\
\bottomrule[1pt]
\end{tabular}
}
\end{center}
\vspace{-4mm}
\end{table}

%% file: secs/4_ex.tex
\section{Experiments}

\subsection{Experimental Setting}
\label{sec:ex_set}
\paragraph{Benchmark.}
We evaluate trajectory-level unlearning on two benchmarks that capture complementary forms of long-horizon agency. \textbf{ALFWorld}~\citep{ALFWorld20} simulates embodied household tasks, requiring an agent to interpret textual observations and execute extended action sequences in a physical environment. It covers six task types: Clean, Heat, Cool, Pick, Pick2, and Look. Detailed statistics are provided in Table~\ref{tab:ALFWorld}. \textbf{WebShop}~\citep{yao2022webshop} simulates realistic online shopping scenarios, where an agent must search, navigate, and select products that satisfy user-specified constraints. Together, the two benchmarks span both embodied and web-based agentic settings, allowing us to assess whether trajectory unlearning generalizes across distinct scenarios.

\vspace{-3mm}
\paragraph{Baselines.}
We evaluate against representative LLM unlearning methods, including {Gradient Ascent} (\textbf{GA})~\citep{yao2024large}, {Direct Preference Optimization} (\textbf{DPO})~\citep{rafailov2023direct}, {Negative Preference Optimization} (\textbf{NPO})~\citep{zhang2024negative} and \textbf{GRPO}~\citep{shao2024deepseekmath}. To ensure a fair comparison, we additionally combine NPO with GRPO (\textbf{NPO+GRPO}) to compensate for the utility degradation introduced by unlearning and better preserve task success rate. See Appendix~\ref{app:baseline} for detailed descriptions of each baseline.

\vspace{-3mm}
\paragraph{Evaluation Metrics.} 
We use \textbf{Exact Match} and \textbf{LLM-as-Judge Similarity} on forget trajectories for the target tasks to evaluate forgetting quality, and \textbf{Success Rate} on both target and untarget tasks to measure model utility. For WebShop, we additionally report \textbf{Score}, a metric defined in the original paper~\citep{yao2022webshop} that measures how well the finally selected product matches the attributes specified in the task instruction. For all metrics, we sample five trajectories for each task and report the average values. For {Exact Match} and {LLM-as-Judge Similarity}, we establish a reference by sampling multiple trajectories from the base model and computing the two metrics against the forget trajectories. This is necessary because even without any unlearning, the model's inherent stochasticity causes independent rollouts to differ from one another, meaning {Exact Match} and {LLM-as-Judge Similarity} between two baseline rollouts are not necessarily 100\%. We refer to this baseline variation as natural variability, which reflects behavioral differences arising solely from the model's intrinsic randomness. Similarly, for Success Rate and Score, we report the base model's task success rate as the reference for utility preservation.

\vspace{-3mm}
\paragraph{Implementation Details.}
We conduct experiments with \textbf{Qwen2.5-3B-Instruct} and \textbf{Qwen2.5-7B-Instruct}~\citep{qwen2.5} as backbone models. Since open-source models generally achieve low task success rates on ALFWorld and WebShop, we adopt post-trained variants as our base models: Qwen2.5-3B-Instruct is trained with SEED~\citep{wu2026seed} on ALFWorld, and Qwen2.5-7B-Instruct with SkillRL~\citep{xia2026skillrl} on WebShop. For \textbf{ALFWorld}, we use 3,553 tasks in total and consider two unlearning settings: a \textit{task-type} setting, where an entire task type is designated as the target tasks, and a \textit{mixed} setting, where target tasks are randomly sampled across all task types. For \textbf{WebShop}, we use 1,000 tasks in total and sample 100 as target tasks. Each task in the target task set has at least one forget trajectory. For on-policy baselines, the forget trajectories are used directly as forget training data. For off-policy baselines, each trajectory is decomposed into multiple prompt-response pairs according to its number of steps, where each prompt includes all preceding actions as context. See Appendix~\ref{app:ex_details} for further details and forget trajectory examples.

\subsection{Experimental Results}
In this section, we present the main performance comparison on both benchmarks, analyze the effect of trajectory completeness, and examine the limitations of existing LLM unlearning methods. Additional experiments, including ablation studies on GiRPO hyperparameters, alternative implementations of GA and NPO, are provided in Appendix~\ref{app:ex}.

\vspace{-3mm}
\paragraph{Main Results.}
\input{tex_tables/results_ALFWorld}
\input{tex_tables/results_WebShop}
Table~\ref{tab:results_ALFWorld} and Table~\ref{tab:results_WebShop} report the performance of all methods on ALFWorld and WebShop, respectively. Across both benchmarks, we observe a clear trade-off between forgetting quality and utility preservation among baselines, which GiRPO resolves most effectively. Taking the Clean target tasks on ALFWorld with Qwen2.5-3B-Instruct as a representative example, \textbf{GA} and \textbf{NPO} achieve strong forgetting, reducing \textit{Exact Match} by 33.7\% and 33.4\% and \textit{LLM-as-Judge Similarity} by 69.6\% and 69.4\%, respectively. However, this comes at the cost of catastrophic utility collapse, with target task \textit{Success Rate} dropping to 0.9\% and 1.2\%, respectively. \textbf{This degradation stems from the nature of these gradient-based knowledge unlearning methods, which often disrupt the model’s original outputs to suppress its ability to reproduce the forgotten knowledge. Such perturbations can interfere with the correct behaviors required to complete the corresponding tasks, thereby severely impairing the model’s task execution capability.}
\textbf{DPO} is more conservative, preserving a higher \textit{Success Rate} but achieving only modest forgetting (\textit{Exact Match} reduced by 14.0\%, \textit{LLM-as-Judge Similarity} by 7.9\%), indicating that it fails to sufficiently suppress the forget trajectories.
Without any unlearning signal, \textbf{GRPO} maintains high task utility but barely exhibits the forgetting effect in terms of \textit{Exact Match} and \textit{LLM-as-Judge Similarity}, confirming that standard policy optimization has no intrinsic forgetting effect.
\textbf{NPO+GRPO} improves the trade-off but remains insufficient. Combining NPO with GRPO recovers some utility while achieving better forgetting. However, the utility drop ($-9.0\%$) remains substantial, and the forgetting quality still lags behind GiRPO, suggesting that simply combining unlearning and policy optimization objectives is suboptimal.
\textbf{GiRPO} achieves the best trade-off. It reduces \textit{Exact Match} by 32.5\% while maintaining a \textit{Success Rate} of 88.4\%, achieving lower trajectory similarity while retaining comparable task success rates than the other baselines. On untarget tasks, GiRPO also preserves \textit{Success Rate} at 88.5\%, demonstrating that unlearning does not interfere with untarget task performance. Overall, GiRPO is the most effective method that achieves strong forgetting quality while preserving both target and untarget task success rates.

\vspace{-3mm}
\paragraph{Completeness of Forget Trajectory.}
In the forget trajectory dataset for `Clean' target tasks, there is one forget trajectory for each of the 650 tasks. Among these, 595 forget trajectories successfully complete the corresponding task (\textit{complete}) and 55 do not (\textit{incomplete}). We analyze the effect of trajectory completeness on unlearning quality and task utility separately.

As shown in the left panel of Figure~\ref{fig:completeness}, the two groups exhibit markedly different forgetting behavior. For the 55 incomplete trajectories, all 5 independently sampled rollouts produce action sequences that differ from the forget trajectory, indicating that incomplete trajectories are naturally easy to forget. In contrast, for complete trajectories, an average of 9 out of 595 trajectories still exactly reproduce the forget action sequence after unlearning, confirming that successfully completed trajectories are significantly harder to unlearn. This asymmetry is intuitive: incomplete trajectories correspond to behaviors the model already struggles to reproduce, whereas complete trajectories represent well-learned, high-reward behavioral patterns that are deeply embedded in the policy.

Regarding task utility in the right panel of Figure~\ref{fig:completeness}, unlearning complete trajectories inevitably impairs task success rate, which drops from 100\% to 92.94\%, as the forget trajectory is the only successful path the model has learned for those tasks. Conversely, unlearning incomplete trajectories improves task success rate from 0\% to 40\%, since suppressing a failed action sequence implicitly steers the model away from suboptimal behaviors and toward successful alternatives.

\begin{figure}
    \vspace{-2mm}
    \centering
    \includegraphics[width=0.385\linewidth]{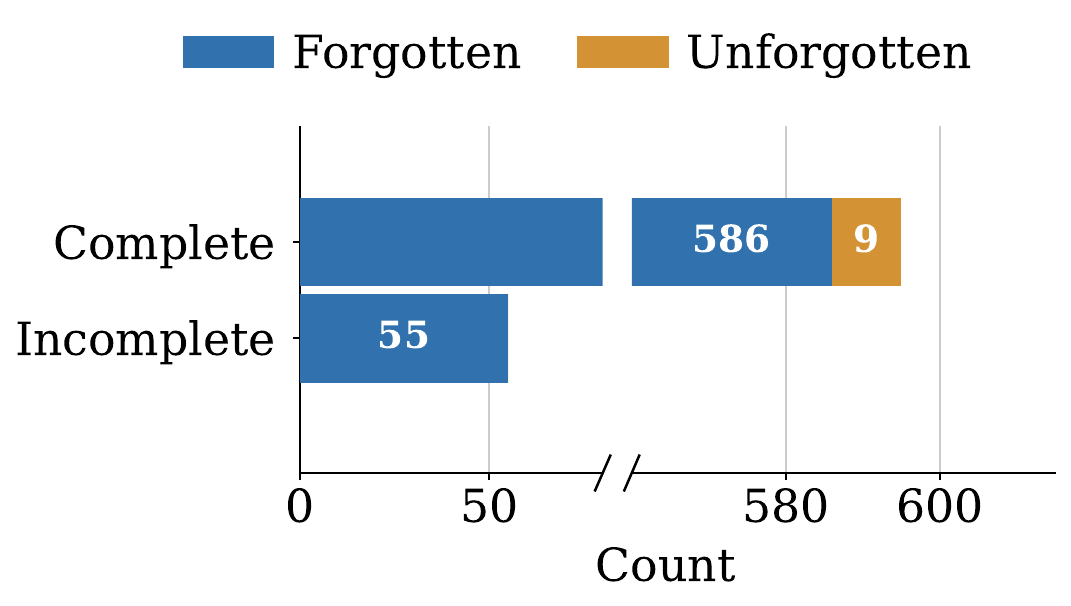}
    \includegraphics[width=0.4\linewidth]{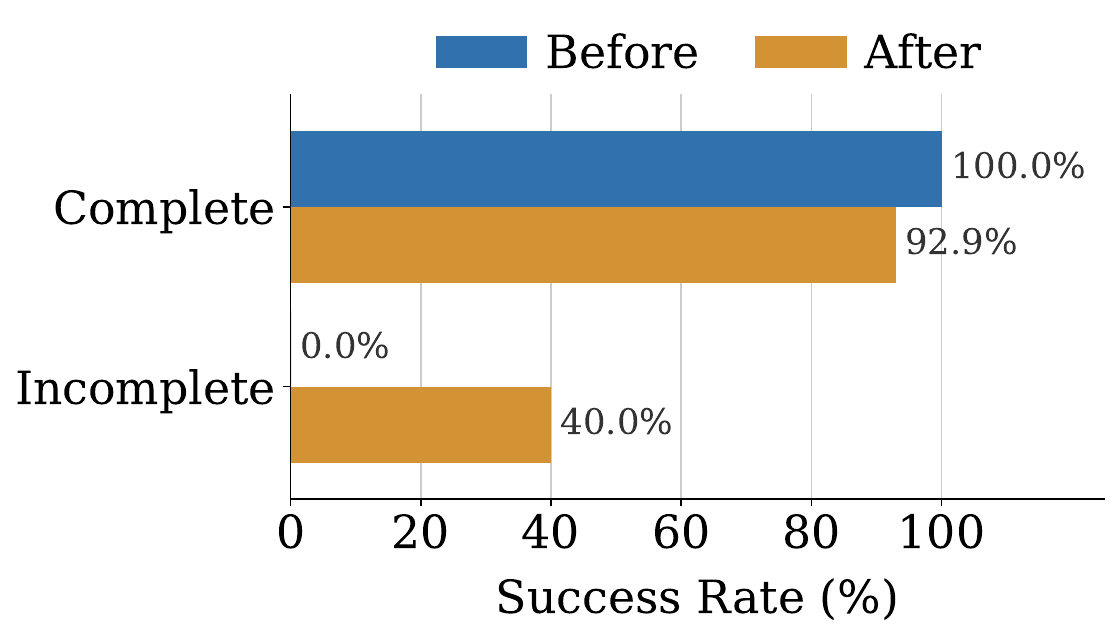}
    \vspace{-4mm}
    \caption{Analysis of unlearning behavior across complete and incomplete forget trajectories.}
    \label{fig:completeness}
    \vspace{-4mm}
\end{figure}

\vspace{-3mm}
\paragraph{Limitations of LLM Unlearning Methods.}
We conduct an ablation study on two representative LLM unlearning methods, GA and NPO, to examine the trade-off between forgetting and utility. Figures~\ref{fig:GA_webshop} and~\ref{fig:NPO_webshop} report their performance on WebShop under varying retain coefficients $\beta$. As $\beta$ increases, both methods better preserve task completion ability, but their forgetting effectiveness on target tasks deteriorates substantially. This exposes an inherent tension in knowledge unlearning objectives: gains in utility come directly at the expense of forgetting, and no choice of $\beta$ achieves both simultaneously.

\input{tex_figs/GA_WebShop}
\input{tex_figs/NPO_WebShop}

%% file: tex_tables/results_ALFWorld.tex
\begin{table*}[t]
\caption{Comparison of different methods on target and untarget tasks across Clean, Heat, Cool, and Mixed task types on \textbf{ALFWorld} with \textbf{Qwen2.5-3B-Instruct}. $\Delta$ denotes the performance difference from the Base Model. The results are average values from five random seeds. Desirable positive impacts on \textit{Success Rate} are highlighted in \sethlcolor{myyellow!30}\hl{yellow}. For \textit{Exact Match} and \textit{LLM-as-Judge Similarity}, darker \sethlcolor{bluecolor1}\hl{blue} shading indicates greater reduction from the base model, reflecting stronger forgetting effectiveness.}
\vspace{-4mm}
\label{tab:results_ALFWorld}
\begin{center}
\resizebox{\textwidth}{!}{
\begin{tabular}{llcccccccc}
\toprule[1pt]
& & \multicolumn{6}{c}{\textbf{Target Tasks}}
& \multicolumn{2}{c}{\textbf{Untarget Tasks}} \\
\cmidrule(lr){3-8}
\cmidrule(lr){9-10}
\textbf{Task} & \textbf{Methods}
& \textbf{Success Rate} $\uparrow$ & $\Delta$
& \textbf{Exact Match} $\downarrow$ & $\Delta$
& \textbf{LLM as Judge} $\downarrow$ & $\Delta$
& \textbf{Success Rate} $\uparrow$ & $\Delta$ \\
\midrule

\multirow{7}{*}{Clean}
& Base Model  & 91.5\% & \textemdash & 33.9\% & \textemdash & 83.0\% & \textemdash & 87.3\% & \textemdash \\
\cmidrule(lr){2-10}
& GA          & 0.9\%  & \small{\textcolor{mygray}{-90.6\%}} & 0.2\%  & \cellcolor{bluecolor1}\small{\textcolor{mygray}{-33.7\%}} & 13.4\% & \cellcolor{bluecolor1}\small{\textcolor{mygray}{-69.6\%}} & 76.6\% & \small{\textcolor{mygray}{-10.7\%}} \\
& DPO         & 83.9\% & \small{\textcolor{mygray}{-7.6\%}}  & 19.9\% & \cellcolor{bluecolor3}\small{\textcolor{mygray}{-14.0\%}} & 75.1\% & \cellcolor{bluecolor3}\small{\textcolor{mygray}{-7.9\%}}  & 82.0\% & \small{\textcolor{mygray}{-5.3\%}} \\
& NPO         & 1.2\%  & \small{\textcolor{mygray}{-90.3\%}} & 0.5\%  & \cellcolor{bluecolor1}\small{\textcolor{mygray}{-33.4\%}} & 13.6\% & \cellcolor{bluecolor1}\small{\textcolor{mygray}{-69.4\%}} & 85.6\% & \small{\textcolor{mygray}{-1.7\%}} \\
& GRPO        & 92.2\% & \cellcolor{myyellow!30}\small{\textcolor{mygray}{0.7\%}}                     & 23.4\% & \cellcolor{bluecolor3}\small{\textcolor{mygray}{-10.5\%}} & 83.0\% & \cellcolor{bluecolor3}\small{\textcolor{mygray}{0.0\%}}   & 90.6\% & \cellcolor{myyellow!30}\small{\textcolor{mygray}{3.3\%}} \\
& NPO+GRPO    & 82.5\% & \small{\textcolor{mygray}{-9.0\%}}  & 12.2\% & \cellcolor{bluecolor2}\small{\textcolor{mygray}{-21.7\%}} & 67.4\% & \cellcolor{bluecolor2}\small{\textcolor{mygray}{-15.6\%}} & 89.0\% & \cellcolor{myyellow!30}\small{\textcolor{mygray}{1.7\%}} \\
\rowcolor{verylightgray}
& \textbf{GiRPO} & \textbf{88.4\%} & \small{\textcolor{mygray}{\textbf{-3.1\%}}} & \textbf{1.4\%} & \cellcolor{bluecolor1}\small{\textcolor{mygray}{\textbf{-32.5\%}}} & \textbf{60.5\%} & \cellcolor{bluecolor1}\small{\textcolor{mygray}{\textbf{-22.5\%}}} & \textbf{88.5\%} & \cellcolor{myyellow!30}\small{\textcolor{mygray}{\textbf{1.2\%}}} \\

\midrule

\multirow{7}{*}{Heat}
& Base Model  & 91.1\% & \textemdash & 26.6\% & \textemdash & 79.2\% & \textemdash & 87.6\% & \textemdash \\
\cmidrule(lr){2-10}
& GA          & 52.9\% & \small{\textcolor{mygray}{-38.2\%}} & 11.1\% & \cellcolor{bluecolor3}\small{\textcolor{mygray}{-15.5\%}} & 54.1\% & \cellcolor{bluecolor1}\small{\textcolor{mygray}{-25.1\%}} & 86.8\% & \small{\textcolor{mygray}{-0.8\%}} \\
& DPO         & 84.1\% & \small{\textcolor{mygray}{-7.0\%}}  & 10.2\% & \cellcolor{bluecolor2}\small{\textcolor{mygray}{-16.4\%}} & 70.2\% & \cellcolor{bluecolor2}\small{\textcolor{mygray}{-9.0\%}}  & 77.2\% & \small{\textcolor{mygray}{-10.4\%}} \\
& NPO         & 1.9\%  & \small{\textcolor{mygray}{-89.1\%}} & 0.4\%  & \cellcolor{bluecolor1}\small{\textcolor{mygray}{-26.2\%}} & 2.0\%  & \cellcolor{bluecolor1}\small{\textcolor{mygray}{-77.2\%}} & 85.0\% & \small{\textcolor{mygray}{-2.6\%}} \\
& GRPO        & 93.0\% & \cellcolor{myyellow!30}\small{\textcolor{mygray}{2.0\%}}                     & 13.5\% & \cellcolor{bluecolor3}\small{\textcolor{mygray}{-13.1\%}} & 75.7\% & \cellcolor{bluecolor3}\small{\textcolor{mygray}{-3.5\%}}  & 87.6\% & \cellcolor{myyellow!30}\small{\textcolor{mygray}{0.0\%}} \\
& NPO+GRPO    & 87.3\% & \small{\textcolor{mygray}{-3.8\%}}  & 9.4\%  & \cellcolor{bluecolor2}\small{\textcolor{mygray}{-17.2\%}} & 73.4\% & \cellcolor{bluecolor3}\small{\textcolor{mygray}{-5.8\%}}  & 86.9\% & \small{\textcolor{mygray}{-0.7\%}} \\
\rowcolor{verylightgray}
& \textbf{GiRPO} & \textbf{87.6\%} & \small{\textcolor{mygray}{\textbf{-3.5\%}}} & \textbf{3.7\%} & \cellcolor{bluecolor1}\small{\textcolor{mygray}{\textbf{-22.9\%}}} & \textbf{60.3\%} & \cellcolor{bluecolor1}\small{\textcolor{mygray}{\textbf{-18.9\%}}} & \textbf{87.1\%} & \small{\textcolor{mygray}{\textbf{-0.5\%}}} \\

\midrule

\multirow{7}{*}{Cool}
& Base Model  & 83.7\% & \textemdash & 19.7\% & \textemdash & 76.4\% & \textemdash & 88.8\% & \textemdash \\
\cmidrule(lr){2-10}
& GA          & 2.2\%  & \small{\textcolor{mygray}{-81.5\%}} & 0.6\%  & \cellcolor{bluecolor1}\small{\textcolor{mygray}{-19.1\%}} & 2.8\%  & \cellcolor{bluecolor1}\small{\textcolor{mygray}{-73.6\%}} & 83.8\% & \small{\textcolor{mygray}{-5.0\%}} \\
& DPO         & 66.4\% & \small{\textcolor{mygray}{-17.3\%}} & 6.0\%  & \cellcolor{bluecolor2}\small{\textcolor{mygray}{-13.7\%}} & 67.5\% & \cellcolor{bluecolor2}\small{\textcolor{mygray}{-8.9\%}}  & 85.9\% & \small{\textcolor{mygray}{-2.9\%}} \\
& NPO         & 2.4\%  & \small{\textcolor{mygray}{-81.3\%}} & 0.8\%  & \cellcolor{bluecolor1}\small{\textcolor{mygray}{-18.9\%}} & 3.3\%  & \cellcolor{bluecolor1}\small{\textcolor{mygray}{-73.1\%}} & 88.0\% & \small{\textcolor{mygray}{-0.8\%}} \\
& GRPO        & 88.6\% & \cellcolor{myyellow!30}\small{\textcolor{mygray}{4.9\%}}                     & 13.1\% & \cellcolor{bluecolor3}\small{\textcolor{mygray}{-6.6\%}}  & 71.1\% & \cellcolor{bluecolor3}\small{\textcolor{mygray}{-5.3\%}}  & 89.0\% & \cellcolor{myyellow!30}\small{\textcolor{mygray}{0.2\%}} \\
& NPO+GRPO    & 83.3\% & \small{\textcolor{mygray}{-0.4\%}}  & 8.5\%  & \cellcolor{bluecolor3}\small{\textcolor{mygray}{-11.2\%}} & 68.6\% & \cellcolor{bluecolor3}\small{\textcolor{mygray}{-7.8\%}}  & 90.6\% & \cellcolor{myyellow!30}\small{\textcolor{mygray}{1.8\%}} \\
\rowcolor{verylightgray}
& \textbf{GiRPO} & \textbf{86.5\%} & \cellcolor{myyellow!30}\small{\textcolor{mygray}{\textbf{2.8\%}}}                     & \textbf{3.6\%} & \cellcolor{bluecolor1}\small{\textcolor{mygray}{\textbf{-16.1\%}}} & \textbf{61.7\%} & \cellcolor{bluecolor1}\small{\textcolor{mygray}{\textbf{-14.7\%}}} & \textbf{88.8\%} & \cellcolor{myyellow!30}\small{\textcolor{mygray}{\textbf{0.0\%}}} \\

\midrule

\multirow{7}{*}{Mixed}
& Base Model  & 82.0\% & \textemdash & 19.0\% & \textemdash & 75.4\% & \textemdash & 88.2\% & \textemdash \\
\cmidrule(lr){2-10}
& GA          & 83.0\% & \cellcolor{myyellow!30}\small{\textcolor{mygray}{1.0\%}}                     & 17.0\% & \cellcolor{bluecolor3}\small{\textcolor{mygray}{-2.0\%}}  & 75.8\% & \cellcolor{bluecolor3}\small{\textcolor{mygray}{0.4\%}}   & 86.9\% & \small{\textcolor{mygray}{-1.3\%}} \\
& DPO         & 75.3\% & \small{\textcolor{mygray}{-6.7\%}}  & 9.2\%  & \cellcolor{bluecolor2}\small{\textcolor{mygray}{-9.8\%}}  & 68.0\% & \cellcolor{bluecolor1}\small{\textcolor{mygray}{-7.4\%}}  & 81.5\% & \small{\textcolor{mygray}{-6.7\%}} \\
& NPO         & 75.1\% & \small{\textcolor{mygray}{-6.9\%}}  & 3.8\%  & \cellcolor{bluecolor1}\small{\textcolor{mygray}{-15.2\%}} & 64.0\% & \cellcolor{bluecolor1}\small{\textcolor{mygray}{-11.4\%}} & 77.0\% & \small{\textcolor{mygray}{-11.2\%}} \\
& GRPO        & 85.3\% & \cellcolor{myyellow!30}\small{\textcolor{mygray}{3.3\%}}                     & 12.0\% & \cellcolor{bluecolor3}\small{\textcolor{mygray}{-7.0\%}}  & 71.4\% & \cellcolor{bluecolor2}\small{\textcolor{mygray}{-4.0\%}}  & 88.5\% & \cellcolor{myyellow!30}\small{\textcolor{mygray}{0.3\%}} \\
& NPO+GRPO    & 85.9\% & \cellcolor{myyellow!30}\small{\textcolor{mygray}{3.9\%}}                     & 11.5\% & \cellcolor{bluecolor3}\small{\textcolor{mygray}{-7.5\%}}  & 72.0\% & \cellcolor{bluecolor3}\small{\textcolor{mygray}{-3.4\%}}  & 86.8\% & \small{\textcolor{mygray}{-1.4\%}} \\
\rowcolor{verylightgray}
& \textbf{GiRPO} & \textbf{87.0\%} & \cellcolor{myyellow!30}\small{\textcolor{mygray}{\textbf{5.0\%}}}                     & \textbf{6.7\%} & \cellcolor{bluecolor1}\small{\textcolor{mygray}{\textbf{-12.3\%}}} & \textbf{70.4\%} & \cellcolor{bluecolor1}\small{\textcolor{mygray}{\textbf{-5.0\%}}} & \textbf{88.6\%}& \cellcolor{myyellow!30}\small{\textcolor{mygray}{\textbf{0.4\%}}} \\

\bottomrule[1pt]
\end{tabular}
}
\end{center}
\end{table*}

%% file: tex_tables/results_WebShop.tex
\begin{table*}[t]
\vspace{-4mm}
\caption{Comparison of different methods on target and untarget tasks on \textbf{WebShop} with \textbf{Qwen2.5-7B-Instruct}. The table follows the same format as Table~\ref{tab:results_ALFWorld}. Additionally, we report the Score \citep{yao2022webshop} to evaluate whether the final selected product matches the task instruction.}
\vspace{-4mm}
\label{tab:results_WebShop}
\begin{center}
\resizebox{\textwidth}{!}{
\begin{tabular}{lcccccccccccc}
\toprule[1pt]
& \multicolumn{8}{c}{\textbf{Target Tasks}}
& \multicolumn{4}{c}{\textbf{Untarget Tasks}} \\
\cmidrule(lr){2-9}
\cmidrule(lr){10-13}
\textbf{Methods}
& \textbf{Score} $\uparrow$ & $\Delta$
& \textbf{Success} $\uparrow$ & $\Delta$
& \textbf{Match} $\downarrow$ & $\Delta$
& \textbf{LLM} $\downarrow$ & $\Delta$
& \textbf{Score} $\uparrow$ & $\Delta$
& \textbf{Success} $\uparrow$ & $\Delta$ \\
\midrule

Base Model
& 76.9\% & \textemdash
& 68.1\% & \textemdash
& 19.4\% & \textemdash
& 86.5\% & \textemdash
& 74.6\% & \textemdash
& 63.6\% & \textemdash \\

\midrule

GA
& 60.3\% & \small{\textcolor{mygray}{-16.6\%}}
& 54.5\% & \small{\textcolor{mygray}{-13.6\%}}
& 6.2\%  & \cellcolor{bluecolor1}\small{\textcolor{mygray}{-13.2\%}}
& 83.1\% & \cellcolor{bluecolor2}\small{\textcolor{mygray}{-3.4\%}}
& 58.6\% & \small{\textcolor{mygray}{-16.0\%}}
& 48.8\% & \small{\textcolor{mygray}{-14.8\%}} \\

DPO
& 72.5\% & \small{\textcolor{mygray}{-4.4\%}}
& 62.8\% & \small{\textcolor{mygray}{-5.3\%}}
& 15.5\% & \cellcolor{bluecolor3}\small{\textcolor{mygray}{-3.9\%}}
& 82.5\% & \cellcolor{bluecolor1}\small{\textcolor{mygray}{-4.0\%}}
& 66.2\% & \small{\textcolor{mygray}{-8.4\%}}
& 50.3\% & \small{\textcolor{mygray}{-13.3\%}} \\

NPO
& 64.4\% & \small{\textcolor{mygray}{-12.6\%}}
& 58.7\% & \small{\textcolor{mygray}{-9.4\%}}
& 16.9\% & \cellcolor{bluecolor3}\small{\textcolor{mygray}{-2.5\%}}
& 86.1\% & \cellcolor{bluecolor3}\small{\textcolor{mygray}{-0.4\%}}
& 64.1\% & \small{\textcolor{mygray}{-10.5\%}}
& 55.3\% & \small{\textcolor{mygray}{-8.3\%}} \\

GRPO
& 72.1\% & \small{\textcolor{mygray}{-4.8\%}}
& 62.4\% & \small{\textcolor{mygray}{-5.7\%}}
& 15.5\% & \cellcolor{bluecolor3}\small{\textcolor{mygray}{-3.9\%}}
& 85.6\% & \cellcolor{bluecolor3}\small{\textcolor{mygray}{-0.9\%}}
& 74.9\% & \cellcolor{myyellow!30}\small{\textcolor{mygray}{0.3\%}}
& 67.3\% & \cellcolor{myyellow!30}\small{\textcolor{mygray}{3.7\%}} \\

NPO+GRPO
& 69.1\% & \small{\textcolor{mygray}{-7.8\%}}
& 63.1\% & \small{\textcolor{mygray}{-5.0\%}}
& 14.2\% & \cellcolor{bluecolor2}\small{\textcolor{mygray}{-5.2\%}}
& 84.8\% & \cellcolor{bluecolor3}\small{\textcolor{mygray}{-1.7\%}}
& 79.1\% & \cellcolor{myyellow!30}\small{\textcolor{mygray}{4.5\%}}
& 65.7\% & \cellcolor{myyellow!30}\small{\textcolor{mygray}{2.1\%}} \\

\rowcolor{verylightgray}
\textbf{GiRPO}
& \textbf{81.4\%} & \cellcolor{myyellow!30}\small{\textcolor{mygray}{\textbf{4.5\%}}}
& \textbf{76.2\%} & \cellcolor{myyellow!30}\small{\textcolor{mygray}{\textbf{8.1\%}}}
& \textbf{3.9\%}  & \cellcolor{bluecolor1}\small{\textcolor{mygray}{\textbf{-15.5\%}}}
& \textbf{79.5\%} & \cellcolor{bluecolor1}\small{\textcolor{mygray}{\textbf{-7.0\%}}}
& \textbf{82.7\%} & \cellcolor{myyellow!30}\small{\textcolor{mygray}{\textbf{8.1\%}}}
& \textbf{66.6\%} & \cellcolor{myyellow!30}\small{\textcolor{mygray}{\textbf{3.0\%}}} \\

\bottomrule[1pt]
\end{tabular}
}
\end{center}
\vspace{-6mm}
\end{table*}

%% file: tex_figs/GA_WebShop.tex
\begin{figure}[htbp]
    \centering
    \vspace{-4mm}
    \includegraphics[width=0.259\linewidth]{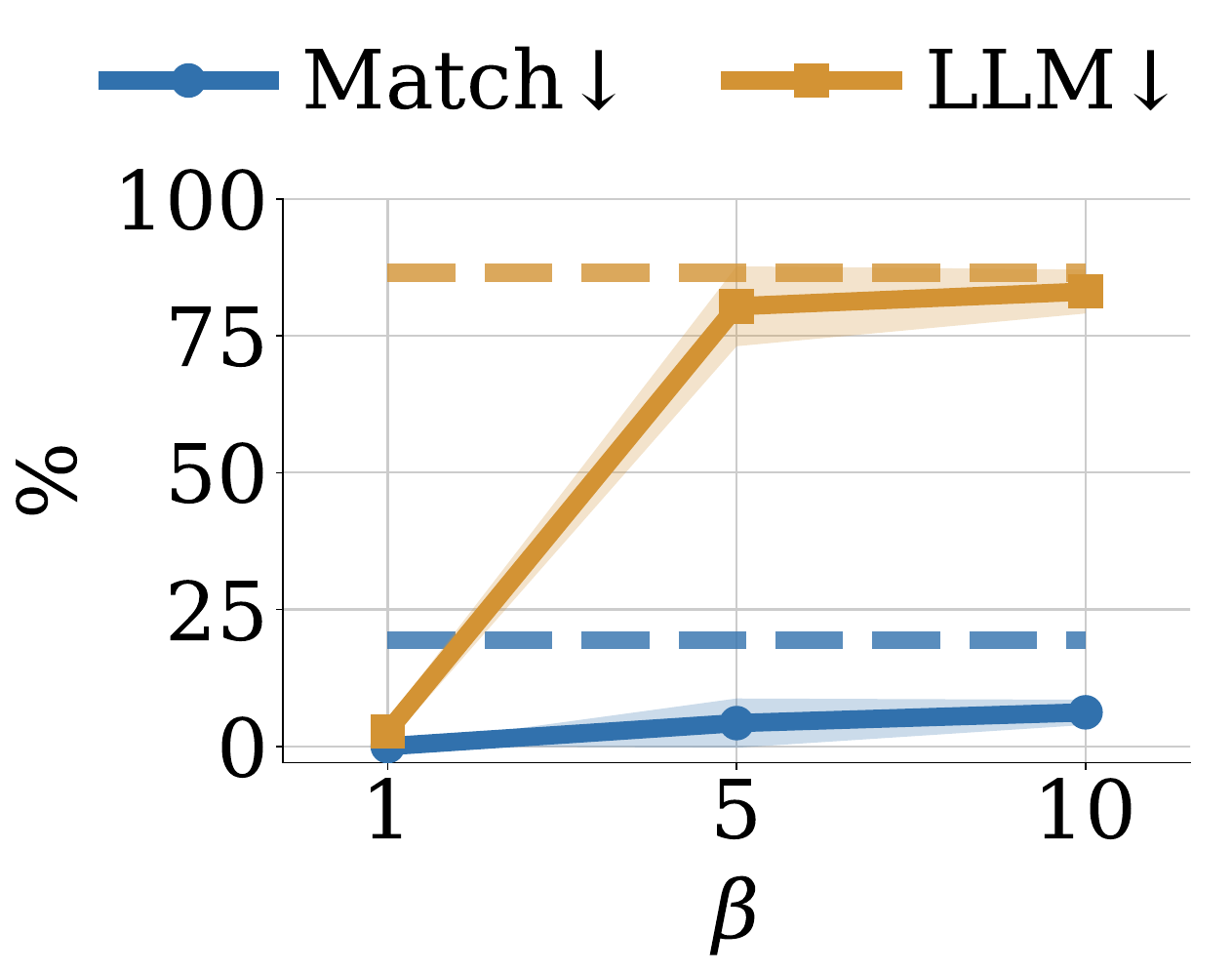}\hspace{2mm}
    \includegraphics[width=0.275\linewidth]{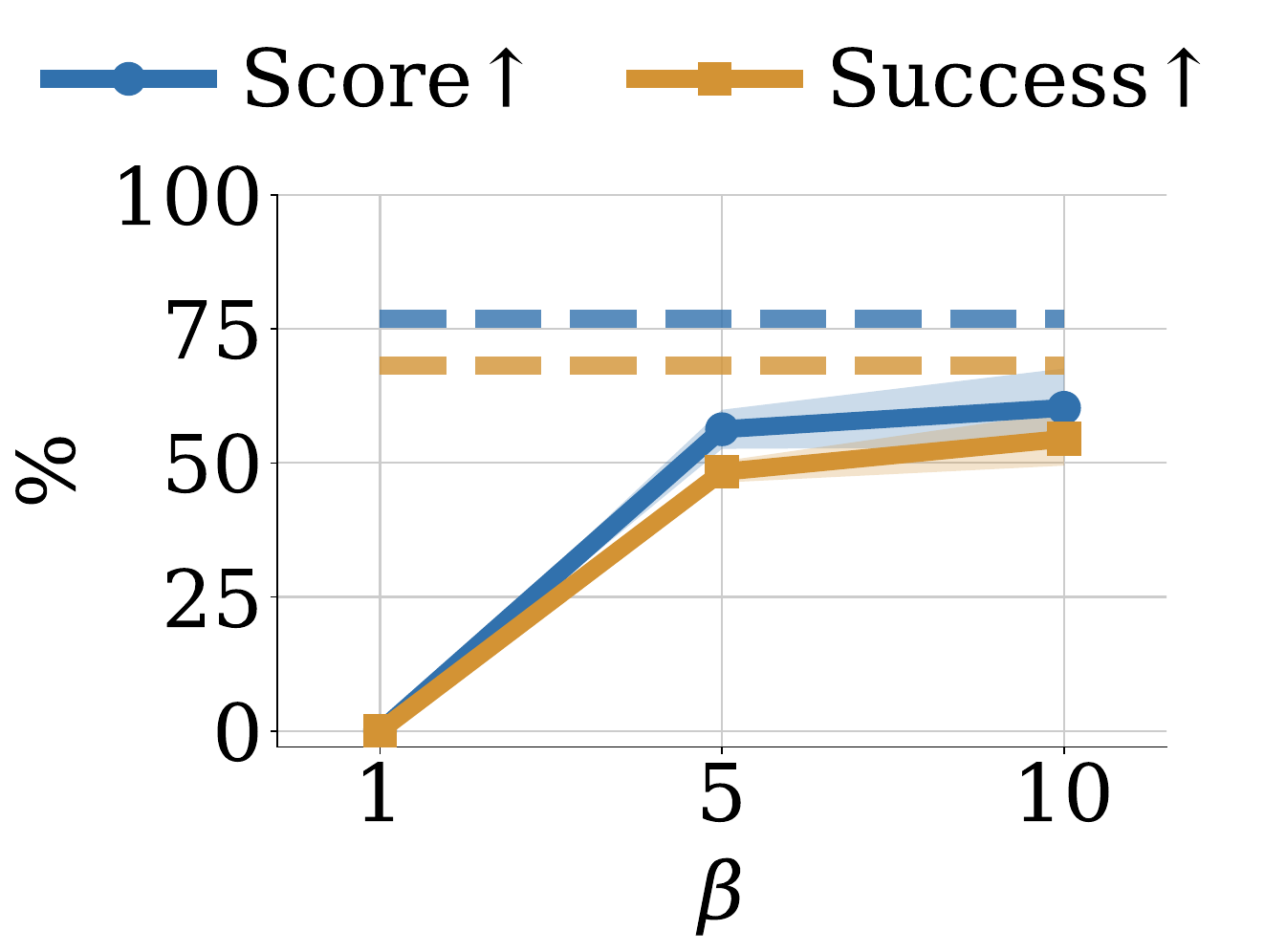}
    \includegraphics[width=0.275\linewidth]{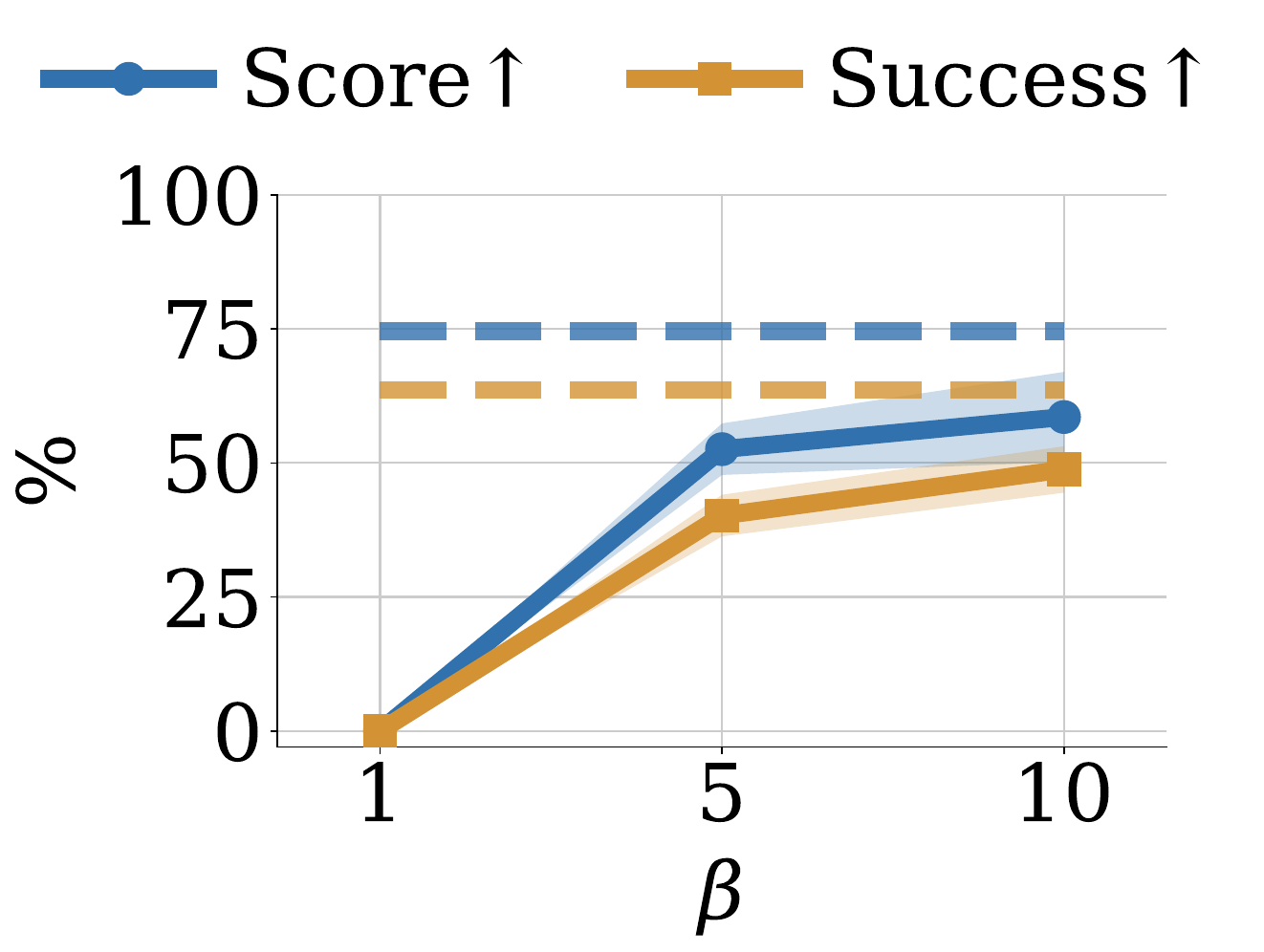}
    \vspace{-4mm}
    \caption{Analysis of GA on WebShop. \textbf{Left:} Forgetting quality on the target tasks, including Exact Match and LLM-As-Judge Similarity. \textbf{Middle:} Task Score and Success Rate on the target tasks. \textbf{Right:} Task Score and Success Rate on the untargeted tasks. The dashed line indicates the base model performance, and the shaded region represents the standard deviation across random seeds.}
    \label{fig:GA_webshop}
\end{figure}
\vspace{-4mm}

%% file: tex_figs/NPO_WebShop.tex
\begin{figure}[htbp]
    \centering
    \includegraphics[width=0.259\linewidth]{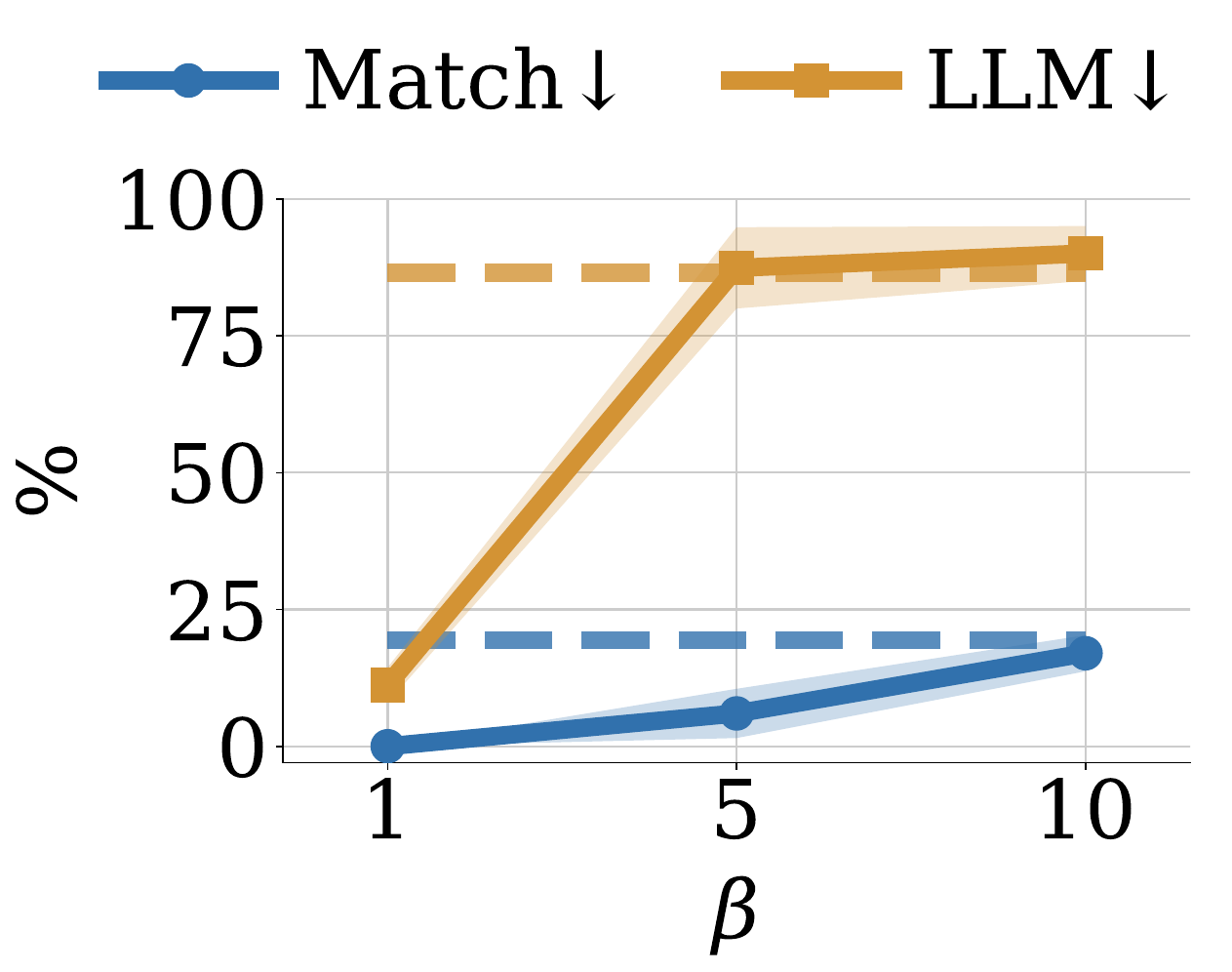}\hspace{2mm}
    \includegraphics[width=0.275\linewidth]{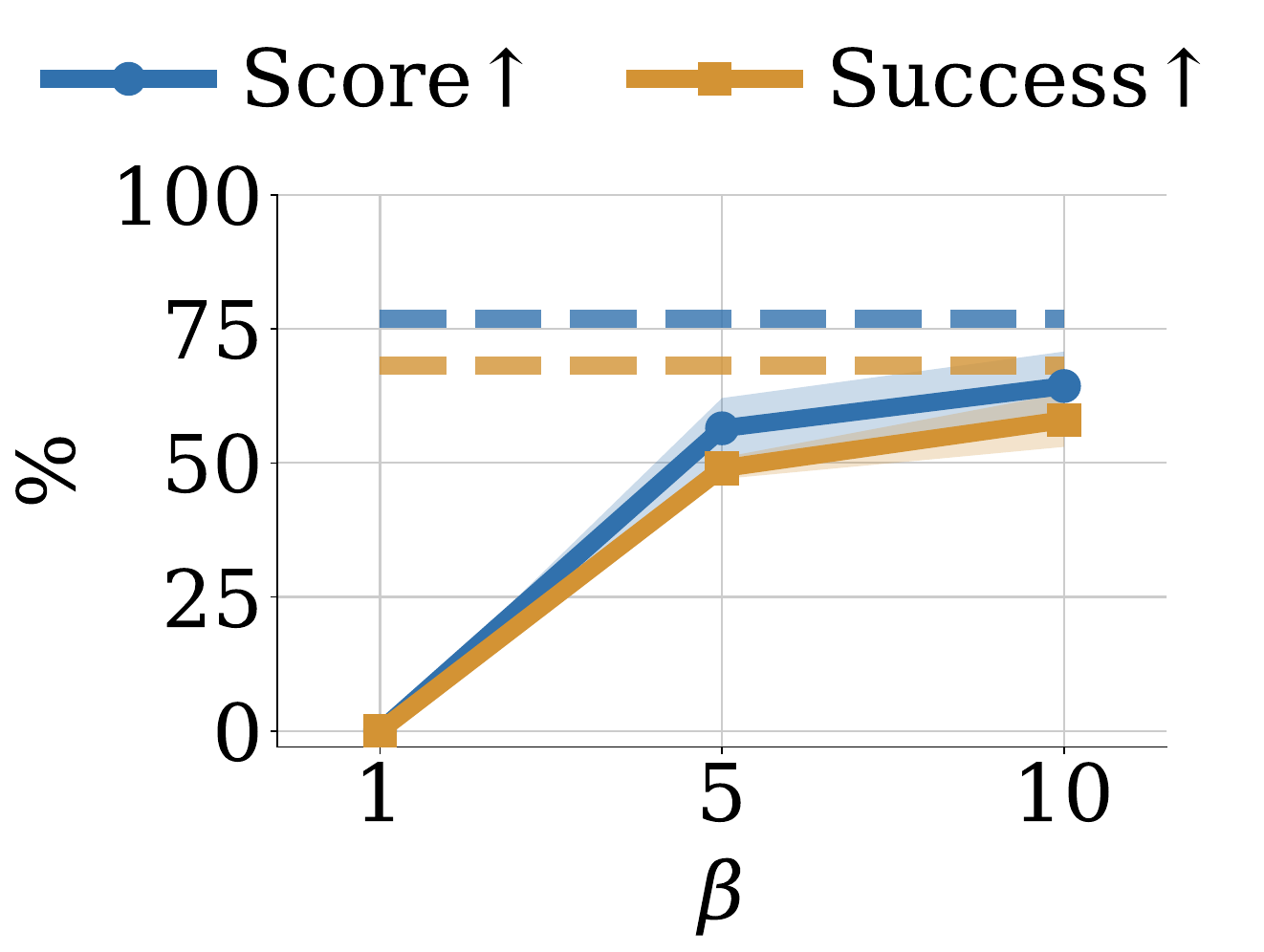}
    \includegraphics[width=0.275\linewidth]{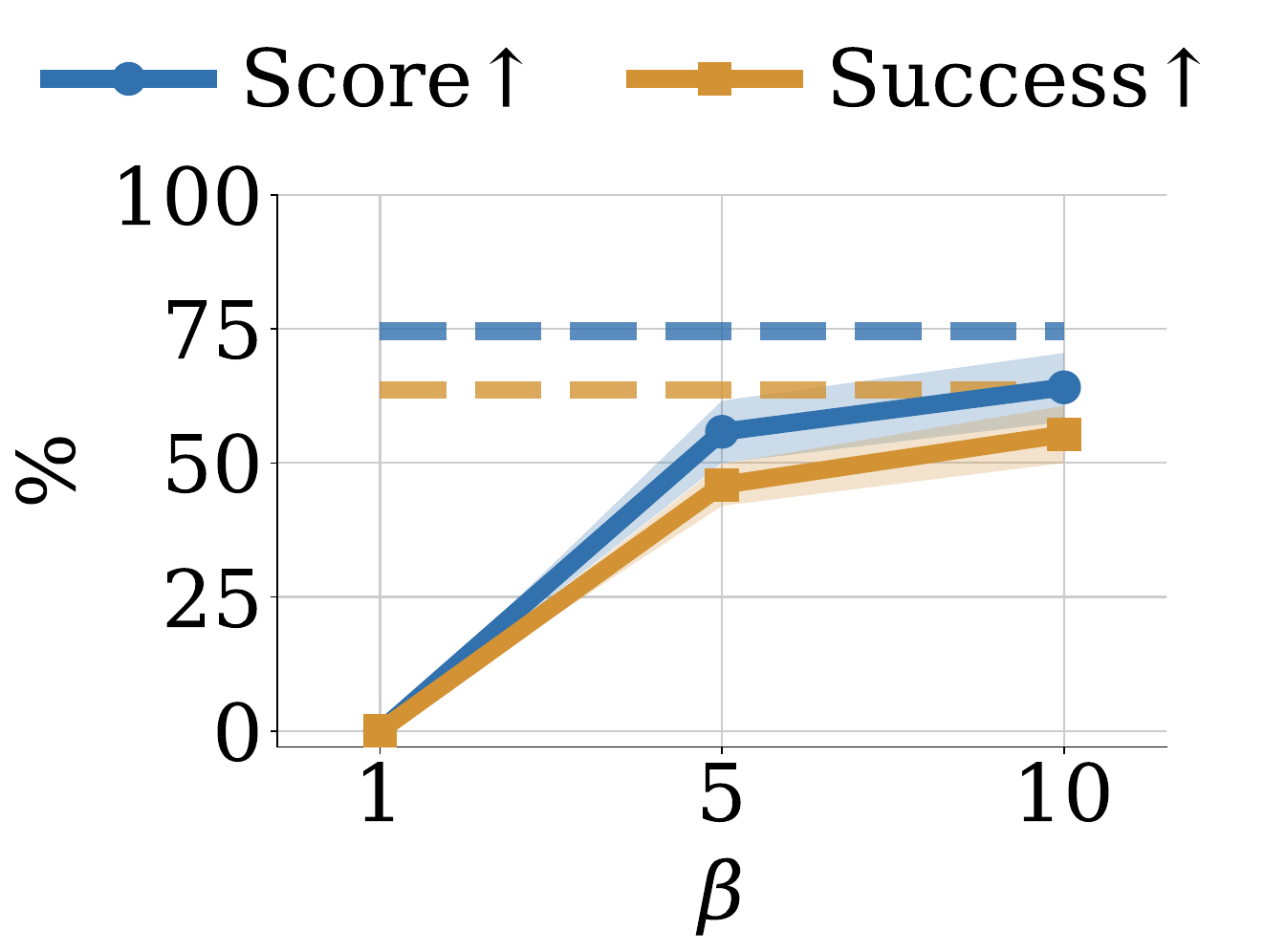}
    \vspace{-4mm}
    \caption{Analysis of NPO on WebShop. The figures follow the same format as Figure~\ref{fig:GA_webshop}.}    \label{fig:NPO_webshop}
\end{figure}

%% file: secs/5_con.tex
\section{Conclusion}

We identify trajectory unlearning as a fundamental and previously overlooked problem in LLM-based agent deployment, where existing knowledge unlearning methods fall short because they ignore the behavioral and long-horizon dependencies inherent in agent trajectories. We formalize the problem, construct benchmarks on ALFWorld and WebShop, and propose method GiRPO. Experiments show that GiRPO achieves strong forgetting while preserving task success rates. We hope this work motivates further research on safe and controllable LLM-based agents.

%% file: secs/6_app.tex
\clearpage
\appendix
\section{Appendix}

\subsection{Additional Experimental Details}
\label{app:ex_details}

\subsubsection{Benchmark}
\label{app:benchmark}
\input{tex_tables/ALFWorld}
The details for the ALFWorld benchmark are shown in Table~\ref{tab:ALFWorld}. The trajectories before and after unlearning are shown in Figures~\ref{fig:traj_example} and~\ref{fig:forget_traj_example}, respectively. An example of the forgotten data for the LLM unlearning methods is shown in Figure~\ref{fig:trajecotry_knowledge}. For WebShop, a trajectory example is shown in Figure~\ref{fig:traj_example_webshop}.

We consider two ways of specifying which trajectories should be forgotten. In the \textbf{\emph{single task-type}} setting, the entire set of trajectories belonging to one ALFWorld task type (e.g., all $650$ \emph{Clean} trajectories) is designated as the target task set, and the remaining five task types serve as the untarget set. In the \textbf{\emph{mixed random-sample}} setting, we instead sample $100$ tasks uniformly at random from the pooled $3{,}553$ tasks spanning \emph{all} six task types, so the target task set is a task-type-agnostic subset scattered across the full task distribution. The remaining $3{,}453$ tasks form the untarget set.

\subsubsection{Baselines}
\label{app:baseline}
\paragraph{Gradient Ascent (GA).} GA~\citep{yao2024large} is one of the earliest and most straightforward unlearning methods. It directly maximizes the loss on the forget trajectories, effectively reversing the gradient descent updates that caused the model to learn the forget knowledge. Formally, the objective is:
\begin{equation}
    \mathcal{L}_{\text{GA}}(\theta) = -\mathcal{L}_{\text{CE}}(\theta; \mathcal{D}_f) + \beta \mathcal{L}_{\text{CE}}(\theta; \mathcal{D}_r),
\end{equation}
where $\mathcal{L}_{\text{CE}}$ denotes the standard cross-entropy loss. $\mathcal{D}_f$ is the forget trajectories of target tasks and $\mathcal{D}_r$ is the trajectories sampled from untarget tasks. In our setting, each forget trajectory $\tau_i \in \mathcal{D}_f$ is decomposed into a sequence of prompt-response pairs, where each prompt contains the full interaction history up to that step. The default value of $\beta$ in our work is set to 10.

\paragraph{Direct Preference Optimization (DPO).} DPO~\citep{rafailov2023direct} is a preference-based alignment method that trains the model to prefer chosen responses over rejected ones without explicit reward modeling. For unlearning, forget responses are treated as rejected and paired with a retain response as chosen. The objective is:
\begin{equation}
    \mathcal{L}_{\text{DPO}}(\theta) = -\mathbb{E}_{(x, y_r, y_f)} \left[ \log \sigma \left( \beta \log \frac{\pi_\theta(y_r \mid x)}{\pi_{\text{ref}}(y_r \mid x)} - \beta \log \frac{\pi_\theta(y_f \mid x)}{\pi_{\text{ref}}(y_f \mid x)} \right) \right],
\end{equation}
where $y_r$ is the chosen (retain) response, $y_f$ is the rejected (forget) response, $\pi_{\text{ref}}$ is the reference policy, and $\beta$ controls the deviation from the reference. In our setting, forget trajectories serve as rejected responses, and the sampled substitute trajectories serve as chosen responses.

\paragraph{Negative Preference Optimization (NPO).} NPO~\citep{zhang2024negative} simplifies the DPO objective by focusing solely on down-weighting the forget responses, without requiring paired chosen responses. The objective is:
\begin{equation}
    \mathcal{L}_{\text{NPO}}(\theta) = -\mathbb{E}_{(x, y_f)} \left[ \log \sigma \left( -a \log \frac{\pi_\theta(y_f \mid x)}{\pi_{\text{ref}}(y_f \mid x)} \right) \right]  + \beta \mathcal{L}_{\text{CE}}(\theta; \mathcal{D}_r),
\end{equation}
where $a$ is the inverse temperature that controls the strength of the preference signal, and $\beta$ balances the NPO loss and the cross-entropy loss on the retain dataset $\mathcal{D}_r$.
Similar to GA and DPO, forget trajectories are decomposed into prompt-response pairs before applying NPO. The default value of $\beta$ in our work is set to 10.

\paragraph{Group Relative Policy Optimization (GRPO).} GRPO~\citep{shao2024deepseekmath} is an on-policy reinforcement learning method that optimizes a policy using relative rewards computed from multiple sampled trajectories for the same task, without requiring a separate value model. In our setting, we apply standard GRPO exclusively to the untarget tasks $\mathcal{G}_{\text{untarget}}$. Specifically, the policy is optimized to maximize the task rewards on the untarget tasks. 

\paragraph{NPO + GRPO.}
To ensure a fair comparison with GiRPO, which inherits the task-utility preservation capability of GRPO, we construct a combined baseline that applies NPO for unlearning and GRPO for utility recovery. Specifically, NPO is applied on the forget trajectory pairs to suppress forget behaviors, while GRPO is simultaneously applied on untarget tasks to maintain task success rate. The combined objective is:
\begin{equation}
    \mathcal{L}_{\text{NPO+GRPO}}(\theta) = \mathcal{L}_{\text{NPO}}(\theta; \mathcal{D}_f) + \alpha \cdot \mathcal{L}_{\text{GRPO}}(\theta; \mathcal{G}_{\text{untarget}}),
\end{equation}
where $\alpha > 0$ is a balancing coefficient and $\mathcal{L}_{\text{GRPO}}$ denotes the standard GRPO objective over retain tasks $\mathcal{G}_untarget$. In our experiments, we set $\alpha = 1$.

\paragraph{Other details.}
We utilize GPT-4o-mini as the judge LLM to evaluate the LLM-as-Judge Similarity metric.
Detailed training configurations for all the baselines, including learning rates, batch sizes, rollout group sizes, and training steps, are provided in our code repository.

\input{secs/3_1_NL}

\subsection{Supplementary Experimental Results}
\label{app:ex}

\subsubsection{Additional Results for NPO}
To enable a fairer comparison between existing LLM knowledge unlearning methods and trajectory unlearning, we design an additional NPO training strategy, denoted as NPO v2. In Table~\ref{tab:results_ALFWorld}, NPO v1 achieves a relatively low \textit{Success Rate} on the target tasks. To address this issue, we sample alternative trajectories for the target tasks and use them as retain data. Specifically, unlike NPO v1, which uses trajectories from untarget tasks as retain data, NPO v2 is trained on both forget trajectories (forget data) and alternative trajectories (retain data) from the target tasks. As shown in Tables~\ref{tab:success_npo_comparison} and~\ref{tab:em_npo_comparison}, NPO v2 improves upon NPO v1 in both \textit{Success Rate} and \textit{Exact Match}. Nevertheless, compared with our GiRPO method in Table~\ref{tab:results_ALFWorld}, GiRPO still maintains a higher \textit{Success Rate} while achieving a larger reduction in \textit{Exact Match}.

\input{tex_tables/NPO_v2}

\subsubsection{Additional Results for GA}
We tried different hyperparameters to get better results and show the best one in the main paper. Here we show the results under different hyperparameters. The results on ALFWorld are shown in Tables~\ref{tab:success_ga_comparison} and ~\ref{tab:em_ga_comparison}. 

\input{tex_tables/GA_ALFWorld}

\subsubsection{Results on On-Policy Distillation}
We additionally experiment with on-policy distillation (OPD)~\citep{opd} applied to target tasks in combination with GRPO on untarget tasks, but find the performance unsatisfactory. Results on ALFWorld with Clean as the target task type are reported in Table~\ref{tab:OPD_ALFWorld}. We report this negative result in the hope that it is useful to future work. We do not claim that on-policy distillation is fundamentally unsuited to trajectory unlearning, only that a vanilla application of it is insufficient, and that adapting it to this setting likely requires non-trivial modification.

\input{tex_tables/OPD_ALFWorld}


\subsubsection{Ablation Study for GiRPO.} 
We investigate the effects of the reward penalty $\delta$ and the advantage clipping threshold $\lambda$. The results for different values of $\delta$ without clipping threshold $\lambda$ are shown in Figure~\ref{fig:ablation_delta}. We observe that varying the reward penalty $\delta$ within 10 has only a minor impact on the unlearning performance. However, when $\delta$ is set to a larger value (e.g., 100), the success rate drops dramatically, significantly compromising the model's utility. This suggests that a relatively small $\delta$ is sufficient to ensure stable training. Intuitively, as long as the forget trajectories are sufficiently penalized such that their probabilities are pushed below those of the alternative trajectories, effective forgetting can be achieved. Therefore, using a small reward penalty is sufficient while avoiding unnecessarily aggressive updates that may compromise training stability. 

Table~\ref{tab:ablation_lambda} reports the distribution of forget trajectory advantages under different values of $\lambda$ with $\delta$ fixed to 1. The clipping threshold effectively bounds the magnitude of the forget advantages, preventing the excessively large parameter updates that would otherwise arise when the raw advantage diverges. This keeps the unlearning signal stable across training steps and avoids the collateral damage to model utility caused by overly aggressive forgetting.

\begin{table}[t]
    \caption{Advantage values of the forget trajectories under different $\lambda$ with fixed $\delta=1$.}
    \label{tab:ablation_lambda}
    \begin{center}
    \begin{tabular}{c|rrr}
        \toprule
        $\lambda$ & Mean & Min & Max \\
        \midrule
        1.0  &-0.9951 &-1.0000 &-0.9751  \\
        2.0  &-1.4792 &-2.0000 &-0.9850  \\
        5.0 &-2.1246 &-4.5285 &-0.9798  \\
        \bottomrule
    \end{tabular}
    \end{center}
\end{table}

\begin{figure}
    \centering
    \includegraphics[width=0.3\linewidth]{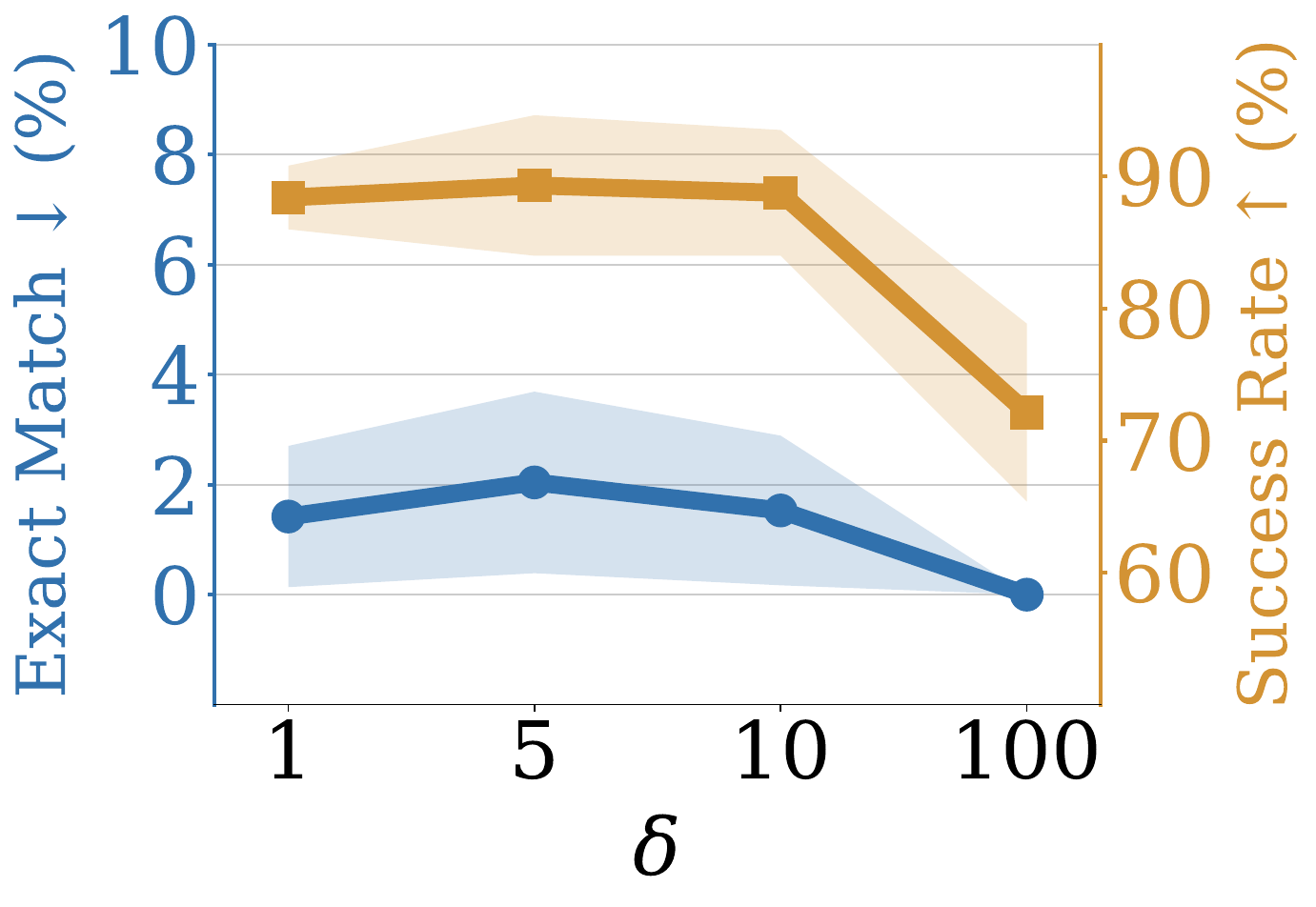}
    \caption{Ablation study for reward penalty $\delta$ with fixed advantage clipping threshold $\lambda$.}
    \label{fig:ablation_delta}
\end{figure}

\input{secs/6_4_theorem}

\subsection{Prompt Templates}
\label{app:prompt_temp}

\subsubsection{Prompt Templates for LLM-as-Judge Evaluation}
\label{app:llm_judge_temp}
We provide the main prompt templates used in LLM-as-Judge evaluation in Figure~\ref{fig:LLM_temp}.

\input{tex_code/LLM}

\subsubsection{Prompt Template for NL}
\label{app:NL_temp}
We provide the main prompt templates used for NL method in Figure~\ref{fig:NL_temp}.
\input{tex_code/NL}

\subsubsection{Attack Prompt Template for NL}
\label{app:NL_attack_temp}
We provide the main prompt templates used for attacking NL method in Figure~\ref{fig:NL_attack_temp}.
\input{tex_code/NL_attack}


\subsection{Future Work}

This work represents a first step toward trajectory-level unlearning in LLM-based agents, and we anticipate many promising directions for future research.

\textbf{Trajectory-level membership inference.} A natural question in unlearning verification is whether a given trajectory has been forgotten. Existing membership inference attacks operate at the token or document level and are not directly applicable to multi-step action sequences. Developing trajectory-level membership inference attacks would provide a principled and adversarial evaluation of forgetting quality, analogous to their role in data privacy verification.

\textbf{Broader agent environments and task types.} Our experiments are conducted on ALFWorld, a text-based household task environment. Extending trajectory unlearning to more complex environments, such as web agents, code execution agents, or multi-modal embodied agents, would test the generality of both the problem formulation and proposed methods.

\textbf{Unlearning under adversarial recovery.} A practically important question is whether a forgotten trajectory can be recovered through fine-tuning or prompt manipulation. Studying the robustness of trajectory unlearning against such recovery attacks would be essential for real-world deployment of safe LLM-based agents.

Overall, we present the framework for trajectory-level unlearning in LLM-based agents, laying the groundwork for a rich and largely unexplored research agenda. The directions outlined above are not limitations of this work, but natural extensions of a new problem space. Our trajectory-level unlearning marks a critical shift in the unlearning paradigm, from suppressing what a model \emph{knows} to controlling what an agent \emph{does}.

\input{tex_code/trajectory}

\input{tex_code/trajectory_unlearned}

\input{tex_code/trajectory_knowledge}
\input{tex_code/webshop}

%% file: tex_tables/ALFWorld.tex
\begin{table}[htbp]
\caption{Details for the ALFWorld Benchmark.}
\label{tab:ALFWorld}
\begin{center}
\begin{tabular}{lcc}
\toprule[1pt]
\textbf{Task Type} & \textbf{Abbr.} & \textbf{Task Number} \\
\midrule
pick clean then place in recep & Clean & 650 \\
pick heat then place in recep & Heat & 459 \\
pick cool then place in recep & Cool & 533 \\
pick and place & Pick & 790 \\
pick two obj and place & Pick2 & 813 \\
look at obj in light & Look & 308 \\
\midrule
\textbf{Total} & -- & \textbf{3553} \\
\bottomrule[1pt]
\end{tabular}
\end{center}
\end{table}

%% file: secs/3_1_NL.tex
\subsection{Limitations in the Existing Trajectory Unlearning Method}
\label{app:NL}
In this section, we discuss the limitations of NL \citep{ye2026secure}, an existing training-free approach for prompt unlearning, and motivate the need for a training-based method for trajectory unlearning. NL's setting focuses on unlearning a specific behavioral rule rather than a specific task trajectory. For example, their agent is required to follow a rule that prioritizes farther options over nearer ones: the agent should first explore distant objects (e.g., ``garbagecan 1'' or ``fridge 1'') before considering more relevant cabinets. This setting differs from ours, where trajectory unlearning aims to remove a specific action sequence associated with a particular task. Nevertheless, we adapt their prompt-based unlearning method, NL, to our trajectory unlearning setting.

Since the authors do not provide publicly available code, we implement the NL method following the description in their paper and construct unlearning prompts specifically for our trajectory unlearning setting. The prompt templates are provided in Appendix~\ref{app:NL_temp} and Appendix~\ref{app:NL_attack_temp}. We evaluate NL under two different settings. The detailed experimental setup is provided in Section~\ref{sec:ex_set}.

\begin{figure}[htbp]
    \centering
    \includegraphics[width=0.32\linewidth]{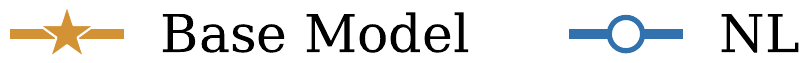}\hspace{2mm}
    
    \includegraphics[width=0.27\linewidth]{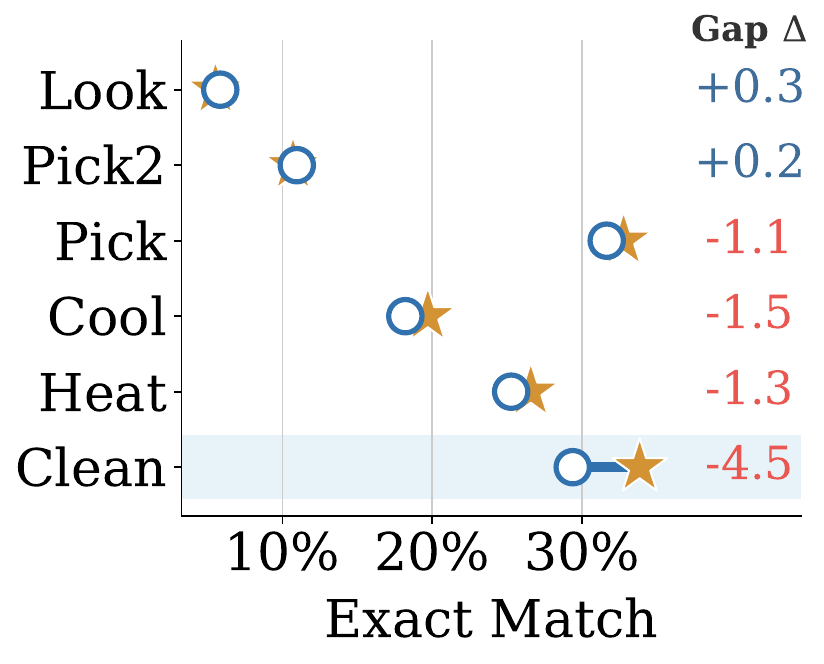}\hspace{2mm}
    \includegraphics[width=0.27\linewidth]{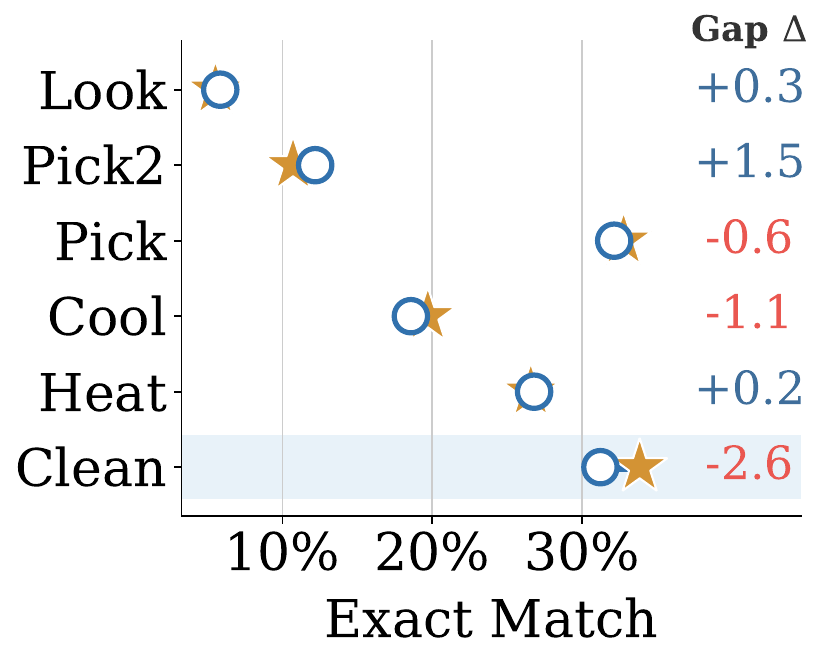}

    \caption{Performance of NL in the single-trajectory unlearning setting. The left panel shows the results of NL, while the right panel shows its performance under a prompt injection attack. NL unlearns poorly in our trajectory problem and degrades further under the attack.}
    \label{fig:NL_single}
\end{figure}
\textbf{Single-trajectory unlearn setting.}
Following the original paper, we construct a separate unlearning prompt for each specific forget task, where the prompt contains the corresponding forget trajectory. We then evaluate the agent on each task and report the average performance across all tasks in Figure~\ref{fig:NL_single}. For the Exact Match metric, a more negative gap indicates better unlearning performance. See Figure~\ref{fig:NL_EM} and \ref{fig:NL_SR} in the Appendix for Success Rate results, and see Table~\ref{tab:NL_single} for specific values.

\begin{figure}[htbp]
    \centering
    \includegraphics[width=0.32\linewidth]{figures/NL_results/label.pdf}\hspace{2mm}

    \includegraphics[width=0.27\linewidth]{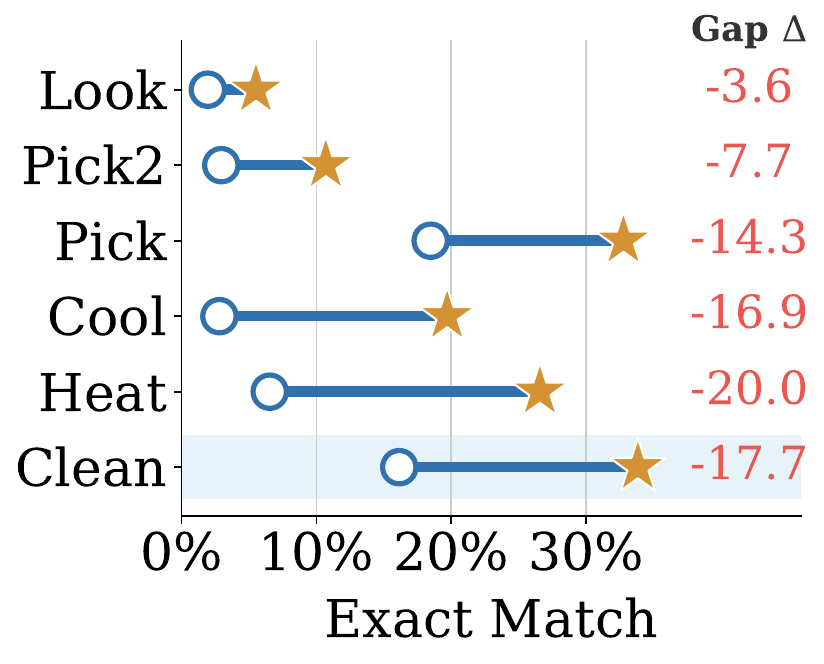}\hspace{2mm}
    \includegraphics[width=0.27\linewidth]{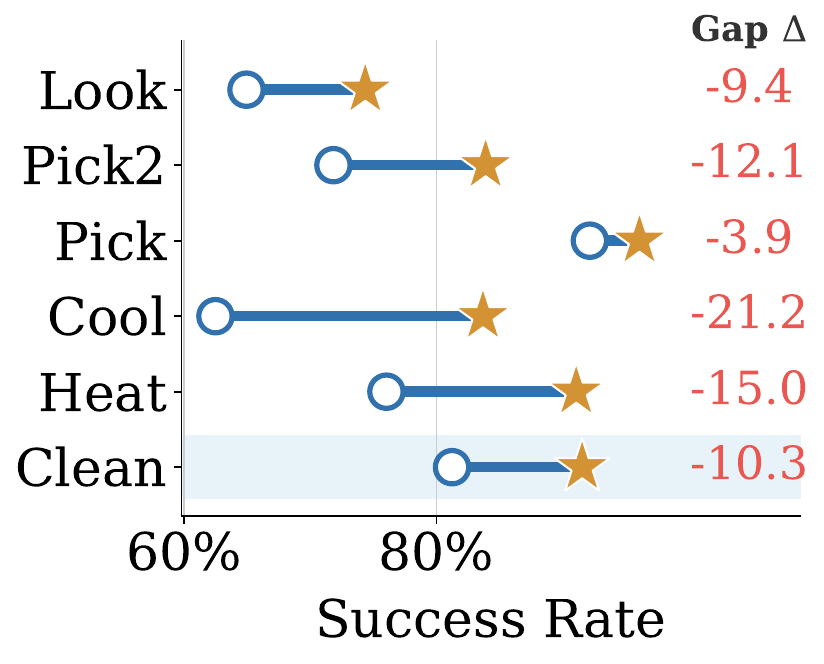}

    \caption{The performance of NL in the multi-trajectory unlearning setting. NL degrades the model’s general performance, resulting in a drop in task success.}
    \label{fig:NL_multi}
\end{figure}
\textbf{Multi-trajectory unlearn setting.}
In a more practical setting, multiple trajectories need to be unlearned simultaneously. We therefore place all forget trajectories into a single system prompt, which is shared across all tasks. However, directly including all forget trajectories in the context is impractical due to the context-length constraint. We thus divide the forget trajectories into groups of 50 and conduct the experiment separately for each group, reporting the average performance across all groups. For the Success Rate metric, a smaller gap, i.e., a value closer to zero, indicates better preservation of model utility. As shown in Table~\ref{tab:NL_multi}, the unlearning performance dramatically deteriorates under this setting.

Overall, our experiments reveal three major limitations of prompt-based trajectory unlearning with NL. 
\ding{202} \textbf{Limited effectiveness on our trajectory unlearning task:} NL does not achieve satisfactory unlearning performance when applied to our trajectory-level setting.
\ding{203} \textbf{Poor scalability to multiple trajectories:} the effectiveness of NL further degrades when multiple forget trajectories are incorporated into a shared system prompt.
\ding{204} \textbf{Vulnerability to prompt attacks:} the unlearning behavior can be easily recovered by a simple prompt attack that instructs the agent to ignore the unlearning system prompt.

These limitations suggest that prompt-based unlearning alone is insufficient for reliable trajectory unlearning. Instead, \textbf{parameter updating is necessary to encode the unlearning objective into the model itself rather than relying solely on an unlearning prompt}.



The complete Exact Match and Success Rate results for the single-trajectory unlearning setting are reported in Tables~\ref{tab:NL_single}, respectively, while the corresponding results for the multi-trajectory unlearning setting are reported in Table~\ref{tab:NL_multi}. The corresponding complete figures for Exact Match and Success Rate are shown in Figures~\ref{fig:NL_EM} and~\ref{fig:NL_SR}, respectively.
\input{tex_tables/NL_single}
\input{tex_tables/NL_multi}
\input{tex_figs/NL}

%% file: tex_tables/NL_single.tex
\begin{table}[t]
\caption{\textit{Exact Match} and \textit{Success Rate} of NL and attacked NL in the single-trajectory unlearning setting.}
\label{tab:NL_single}
\begin{center}
\begin{tabular}{l|c|cc|cc}
\toprule[1pt]
\textbf{Task Type} &
\textbf{Base Model} &
\textbf{NL} &
\textbf{$\Delta$} &
\textbf{Attacked NL} &
\textbf{$\Delta$} \\
\midrule
\rowcolor{verylightgray}
\multicolumn{6}{c}{\textbf{Exact Match} $\downarrow$} \\
\midrule
Clean & 33.85\% & 29.38\% & -4.47\% & 31.23\% & -2.62\% \\
Heat & 26.58\% & 25.27\% & -1.31\% & 26.80\% &  0.22\% \\
Cool & 19.70\% & 18.20\% & -1.50\% & 18.57\% & -1.13\% \\
Pick & 32.78\% & 31.65\% & -1.13\% & 32.15\% & -0.63\% \\
Pick2 & 10.70\% & 10.95\% &  0.25\% & 12.18\% &  1.48\% \\
Look &  5.52\% &  5.84\% &  0.32\% &  5.84\% &  0.32\% \\
\midrule

\rowcolor{verylightgray}
\multicolumn{6}{c}{\textbf{Success Rate} $\uparrow$} \\
\midrule
Clean & 91.54\% & 92.15\% & 0.61\% & 91.69\% & 0.15\% \\
Heat & 91.07\% & 93.68\% & 2.61\% & 91.50\% & 0.43\% \\
Cool & 83.68\% & 83.68\% & 0.00\% & 82.74\% & -0.94\% \\
Pick & 96.08\% & 96.58\% & 0.50\% & 96.08\% & 0.00\% \\
Pick2 & 83.89\% & 83.15\% & -0.74\% & 82.90\% & -0.99\% \\
Look & 74.35\% & 80.52\% & 6.17\% & 75.00\% & 0.65\% \\
\bottomrule[1pt]
\end{tabular}
\end{center}
\end{table}

%% file: tex_tables/NL_multi.tex
\begin{table}[t]
\caption{Task success rate and exact action-sequence match compared with the baseline.}
\label{tab:NL_multi}
\begin{center}
\begin{tabular}{l|ccc|ccc}
\toprule[1pt]
\textbf{Task Type} 
& \multicolumn{3}{c|}{\textbf{Success Rate $\uparrow$}}
& \multicolumn{3}{c}{\textbf{Exact Match $\downarrow$}} \\
\cmidrule(lr){2-4} \cmidrule(lr){5-7}
& \textbf{Base Model} & \textbf{NL} & \textbf{$\Delta$}
& \textbf{Base Model} & \textbf{NL} & \textbf{$\Delta$} \\
\midrule
Clean & 91.54\% & 81.23\% & -10.31\% & 33.85\% & 16.15\% & -17.70\% \\
Heat                  & 91.07\% & 76.03\% & -15.04\% & 26.58\% & 6.54\%  & -20.04\% \\
Cool                  & 83.68\% & 62.48\% & -21.20\% & 19.70\% & 2.81\%  & -16.89\% \\
Pick                  & 96.08\% & 92.15\% & -3.93\%  & 32.78\% & 18.48\% & -14.30\% \\
Pick2                 & 83.89\% & 71.83\% & -12.06\% & 10.70\% & 2.95\%  & -7.75\%  \\
Look                  & 74.35\% & 64.94\% & -9.41\%  & 5.52\%  & 1.95\%  & -3.57\%  \\
\midrule
\textbf{Overall}      & \textbf{88.07\%} & \textbf{76.61\%} & \textbf{-11.46\%}
& \textbf{22.80\%} & \textbf{9.18\%} & \textbf{-13.62\%} \\
\bottomrule[1pt]
\end{tabular}
\end{center}
\end{table}

%% file: tex_figs/NL.tex
\begin{figure}
    \centering

\includegraphics[width=0.26\linewidth]{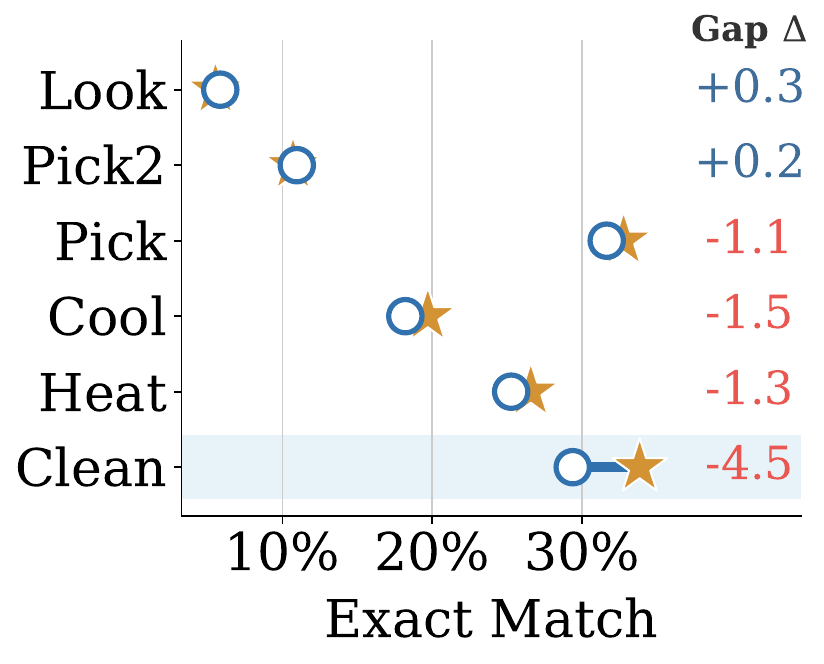}
\hspace{2mm}
\includegraphics[width=0.26\linewidth]{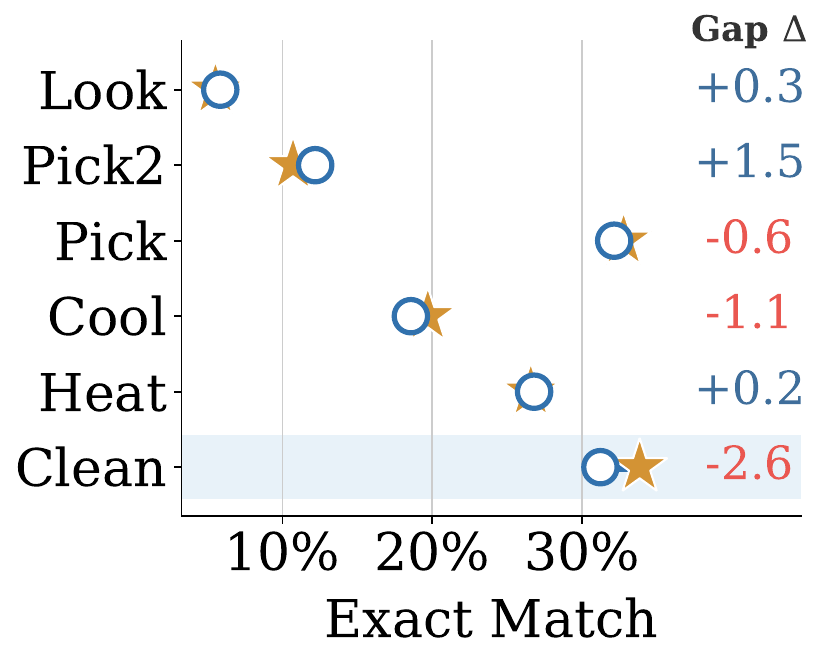}
\hspace{2mm}
\includegraphics[width=0.26\linewidth]{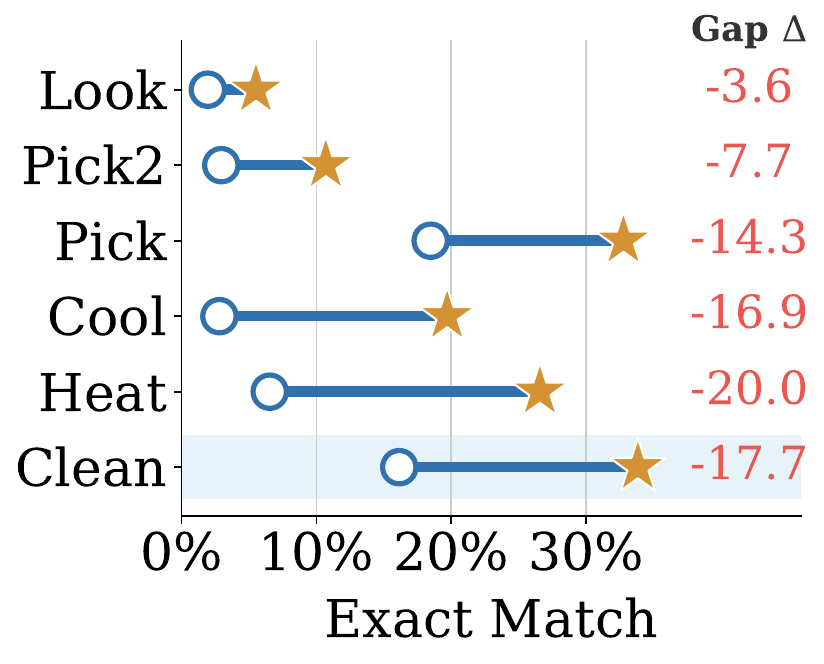}

\caption{Exact Match results for NL in different unlearning settings. Left: NL under single-trajectory unlearning. Middle: attacked NL under single-trajectory unlearning. Right: NL under multi-trajectory unlearning setting.}
\label{fig:NL_EM}
\end{figure}

\begin{figure}
    \centering

\includegraphics[width=0.26\linewidth]{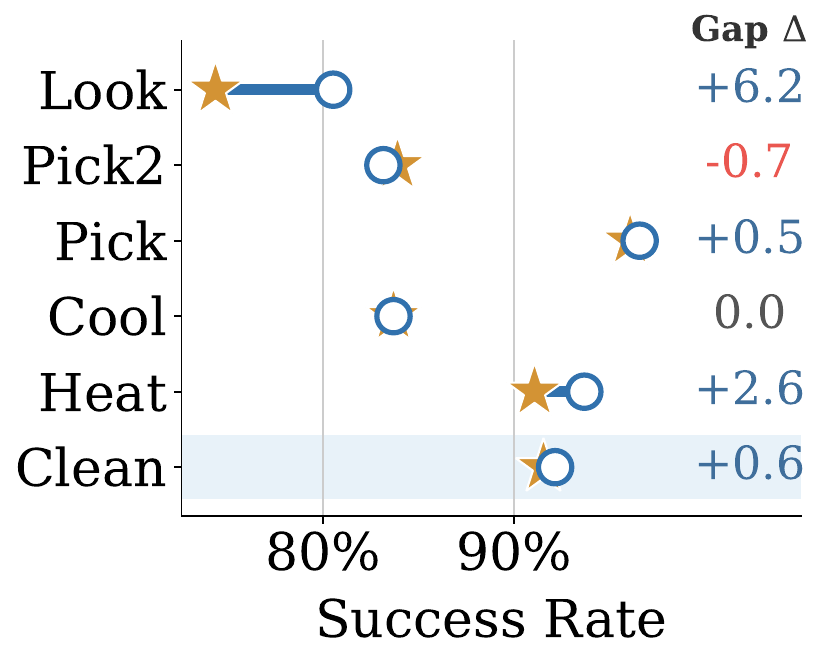}
\hspace{2mm}
\includegraphics[width=0.26\linewidth]{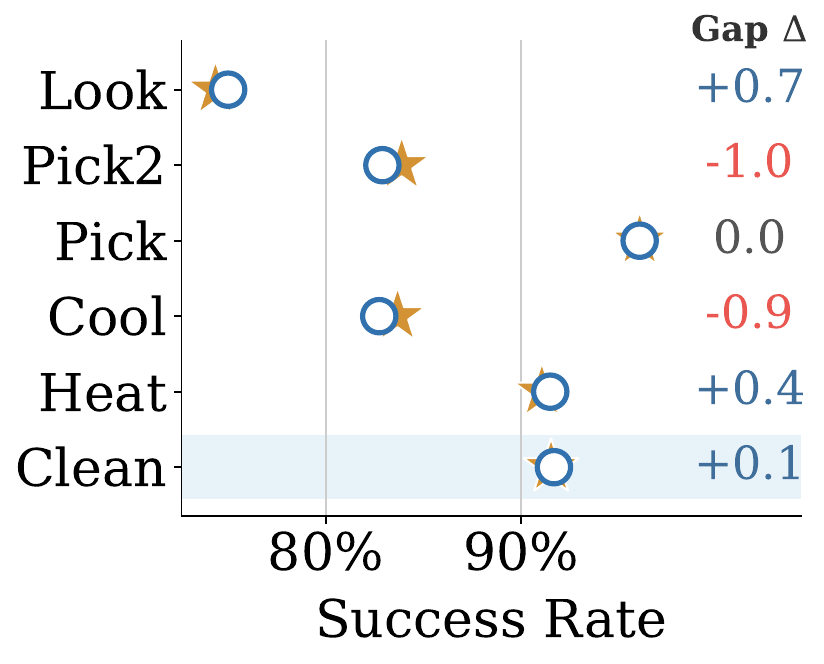}
\hspace{2mm}
\includegraphics[width=0.26\linewidth]{figures/NL_results/Multi_NL_SR.pdf}

\caption{Success Rate results for NL in different unlearning settings. Left: NL under single-trajectory unlearning. Middle: attacked NL under single-trajectory unlearning. Right: NL under multi-trajectory unlearning setting.}
\label{fig:NL_SR}
\end{figure}

%% file: tex_tables/NPO_v2.tex
\begin{table*}[t]
\caption{Comparison of task success rates between NPO v1 and NPO v2. $\Delta$ denotes the performance difference from the Base Model.}
\label{tab:success_npo_comparison}
\begin{center}
\resizebox{0.8\textwidth}{!}{
\begin{tabular}{lccccc}
\toprule[1pt]
\textbf{Task Type}
& \textbf{Base Model}
& \multicolumn{2}{c}{\textbf{NPO v1}}
& \multicolumn{2}{c}{\textbf{NPO v2}} \\
\cmidrule(lr){3-4}
\cmidrule(lr){5-6}
& & \textbf{Success Rate} & $\Delta$ & \textbf{Success Rate} & $\Delta$ \\
\midrule

Clean (Target Task)
& 91.54\% & 1.23\% & -90.31\% & 87.85\% & -3.69\% \\

Heat
& 91.07\% & 93.25\% & +2.18\% & 92.59\% & +1.52\% \\

Cool
& 83.68\% & 77.30\% & -6.38\% & 63.60\% & -20.08\% \\

Pick
& 96.08\% & 95.82\% & -0.26\% & 92.53\% & -3.55\% \\

Pick2
& 83.89\% & 82.66\% & -1.23\% & 73.43\% & -10.46\% \\

Look
& 74.35\% & 69.81\% & -4.54\% & 70.45\% & -3.90\% \\

\bottomrule[1pt]
\end{tabular}
}
\end{center}
\end{table*}

\begin{table*}[t]
\caption{Comparison of exact-match performance between NPO v1 and NPO v2. $\Delta$ denotes the performance difference from the Base Model.}
\label{tab:em_npo_comparison}
\begin{center}
\resizebox{0.8\textwidth}{!}{
\begin{tabular}{lccccc}
\toprule[1pt]
\textbf{Task Type}
& \textbf{Base Model}
& \multicolumn{2}{c}{\textbf{NPO v1}}
& \multicolumn{2}{c}{\textbf{NPO v2}} \\
\cmidrule(lr){3-4}
\cmidrule(lr){5-6}
& & \textbf{Exact Match} & $\Delta$ & \textbf{Exact Match} & $\Delta$ \\
\midrule

Clean (Target Task)
& 33.85\% & 0.46\% & -33.39\% & 12.00\% & -21.85\% \\

\bottomrule[1pt]
\end{tabular}
}
\end{center}
\end{table*}






%% file: tex_tables/GA_ALFWorld.tex
\begin{table*}[t]
\caption{Comparison of Success Rate between GA ($\beta=1.0$) and GA ($\beta=10.0$). $\Delta$ denotes the performance difference from the base model.}
\label{tab:success_ga_comparison}
\begin{center}
\resizebox{0.8\textwidth}{!}{
\begin{tabular}{lccccc}
\toprule[1pt]
\textbf{Task Type}
& \textbf{Base Model}
& \multicolumn{2}{c}{\textbf{GA ($\gamma=1.0$)}}
& \multicolumn{2}{c}{\textbf{GA ($\gamma=10.0$)}} \\
\cmidrule(lr){3-4}
\cmidrule(lr){5-6}
& & \textbf{Success Rate} & $\Delta$ & \textbf{Success Rate} & $\Delta$ \\
\midrule

Clean (Target Task)
& 91.50\% & 0.00\% & -91.50\% & 0.92\% & -90.58\% \\

Heat
& 91.10\% & 6.75\% & -84.35\% & 73.20\% & -17.90\% \\

Cool
& 83.70\% & 5.44\% & -78.26\% & 63.23\% & -20.47\% \\

Pick
& 96.10\% & 80.51\% & -15.59\% & 91.77\% & -4.33\% \\

Pick2
& 83.90\% & 21.53\% & -62.37\% & 75.40\% & -8.50\% \\

Look
& 74.40\% & 17.21\% & -57.19\% & 68.83\% & -5.57\% \\

\bottomrule[1pt]
\end{tabular}
}
\end{center}
\end{table*}

\begin{table*}[t]
\caption{Comparison of Exact Match performance between GA ($\beta=1.0$) and GA ($\beta=10.0$). $\Delta$ denotes the performance difference from the Base Model.}
\label{tab:em_ga_comparison}
\begin{center}
\resizebox{0.8\textwidth}{!}{
\begin{tabular}{lccccc}
\toprule[1pt]
\textbf{Task Type}
& \textbf{Base Model}
& \multicolumn{2}{c}{\textbf{GA ($\gamma=1.0$)}}
& \multicolumn{2}{c}{\textbf{GA ($\gamma=10.0$)}} \\
\cmidrule(lr){3-4}
\cmidrule(lr){5-6}
& & \textbf{Exact Match} & $\Delta$ & \textbf{Exact Match} & $\Delta$ \\
\midrule

Clean (Target Task)
& 33.90\% & 0.00\% & -33.90\% & 0.20\% & -33.70\% \\






\bottomrule[1pt]
\end{tabular}
}
\end{center}
\end{table*}

%% file: tex_tables/OPD_ALFWorld.tex
\begin{table*}[htbp]
\caption{The results for OPD.}
\label{tab:OPD_ALFWorld}
\begin{center}
\resizebox{\textwidth}{!}{
\begin{tabular}{lcccccccc}
\toprule[1pt]
& \multicolumn{6}{c}{\textbf{Target Tasks}}
& \multicolumn{2}{c}{\textbf{Untarget Tasks}} \\
\cmidrule(lr){2-7}
\cmidrule(lr){8-9}
\textbf{Methods}
& \textbf{Success Rate} $\uparrow$ & $\Delta$
& \textbf{Exact Match} $\downarrow$ & $\Delta$
& \textbf{LLM as Judge} $\downarrow$ & $\Delta$
& \textbf{Success Rate} $\uparrow$ & $\Delta$ \\
\midrule

Base Model  & 91.50\% & \textemdash & 33.90\% & \textemdash & 83.00\% & \textemdash & 87.30\% & \textemdash \\
\midrule

OPD          & 81.10\%  & -10.40\% & 15.12\%  & -18.78\% & 65.70\% & -17.30\% & 88.10\% & 0.80\% \\

\bottomrule[1pt]
\end{tabular}
}
\end{center}
\end{table*}

%% file: secs/6_4_theorem.tex

\subsection{Theoretical Analysis of Isolated Normalization in GiRPO}
\label{app:theory}

We formalize the claim made in Section~\ref{sec:method} and illustrated in Table~\ref{tab:advantage_comparison}. Including the injected forget trajectory in the group normalization statistics, as vanilla GRPO would, introduces a systematic and non-vanishing bias into the advantage of every real rollout in the group, while the isolated normalization used by GiRPO does not. We also justify the necessity of the lower clip $\lambda$.

\paragraph{Setup.} Fix a target task $g_i \in \mathcal{G}_{\text{target}}$ and drop the task subscript $i$ for brevity. Let $r^{1}, \ldots, r^{G}$ denote the episode-level rewards of the $G$ real rollouts retained in the group, and let $\mu$ and $\sigma$ be the statistics computed over these real rollouts only, as in Eq.~\eqref{eq:isolated_stats}. Let $r^{f} = \min_{l} r^{l} - \delta$ with $\delta > 0$ be the reward assigned to the injected forget trajectory by Eq.~\eqref{eq:reward_penalty}, and define $\Delta := r^{f} - \mu$. Since $\min_{l} r^{l} \le \mu$ by definition of the mean,
\begin{equation}
\label{eq:delta_bound}
\Delta = r^{f} - \mu \le -\delta, \qquad \text{so} \quad |\Delta| \ge \delta .
\end{equation}
We compare two ways of computing advantages for the group $\mathcal{T} = \{r^{1}, \ldots, r^{G}, r^{f}\}$. Under \emph{isolated normalization}, which is what GiRPO uses, $\mu$ and $\sigma$ come from the $G$ real rewards only, the real advantages are $\hat{A}^{l} = (r^{l} - \mu)/\sigma$ as in Eq.~\eqref{eq:adv_real}, and the forget advantage is $\hat{A}^{f} = \max\bigl((r^{f} - \mu)/\sigma,\, -\lambda\bigr)$ as in Eq.~\eqref{eq:adv_forget}. Under \emph{naive inclusion}, which is vanilla GRPO applied to the augmented group, the statistics $\mu'$ and $\sigma'$ are computed over all $G+1$ rewards, giving real advantages $A'^{l} = (r^{l} - \mu')/\sigma'$. We take $\epsilon \to 0$ throughout for clarity, and all statements below hold with an $O(\epsilon)$ correction otherwise.

\begin{proposition}[Isolated normalization is exact]
\label{prop:exact}
For every $l = 1, \ldots, G$, the advantage $\hat{A}^{l}$ equals the advantage that vanilla GRPO would compute on the group $\{r^{1}, \ldots, r^{G}\}$ in the absence of any forget trajectory injection.
\end{proposition}

\begin{proof}
Immediate, since $\mu$ and $\sigma$ in the isolated estimator are by construction the group statistics of the real rollouts alone.
\end{proof}

Proposition~\ref{prop:exact} confirms the third row of Table~\ref{tab:advantage_comparison}. Isolating the normalization statistics is not only a stabilizing heuristic, it recovers the untouched GRPO advantage for every real trajectory. The outer factor $1/|\mathcal{T}_i|$ in Eq.~\eqref{eq:objective} rescales the whole group by a task-dependent constant, which leaves the relative weighting among real rollouts unchanged.

\begin{proposition}[Mean shift under naive inclusion]
\label{prop:mean}
The pooled mean satisfies $\mu' = \mu + \Delta/(G+1)$, hence $\mu - \mu' \ge \delta/(G+1) > 0$.
\end{proposition}

\begin{proof}
By definition $\mu' = (G\mu + r^{f})/(G+1) = \mu + (r^{f} - \mu)/(G+1) = \mu + \Delta/(G+1)$. The bound follows from Eq.~\eqref{eq:delta_bound}.
\end{proof}

\begin{proposition}[Variance inflation under naive inclusion]
\label{prop:var}
The pooled variance satisfies
\begin{equation}
\label{eq:pooled_var}
\sigma'^{2} = \frac{G}{G+1}\,\sigma^{2} + \frac{G}{(G+1)^{2}}\,\Delta^{2}.
\end{equation}
\end{proposition}

\begin{proof}
Treat the $G$ real rewards, with mean $\mu$ and sum of squares $G\sigma^{2}$, and the single point $r^{f}$, with sum of squares $0$, as two groups being pooled. The standard pooled sum of squares identity gives $\mathrm{SS} = G\sigma^{2} + \frac{G \cdot 1}{G+1}(\mu - r^{f})^{2} = G\sigma^{2} + \frac{G}{G+1}\Delta^{2}$, and $\sigma'^{2} = \mathrm{SS}/(G+1)$ gives Eq.~\eqref{eq:pooled_var}.
\end{proof}

Since $\Delta^{2} \ge \delta^{2}$ by Eq.~\eqref{eq:delta_bound}, Proposition~\ref{prop:var} shows that $\sigma'^{2}$ is inflated by an amount that grows with the penalty $\delta$. Even when $\sigma = 0$, meaning all real rollouts receive the same reward, naive inclusion produces $\sigma' > 0$ purely as an artifact of the injected forget trajectory.

\begin{theorem}[Bias decomposition of the naive advantage]
\label{thm:bias}
For every real trajectory $l$,
\begin{equation}
\label{eq:bias_decomp}
A'^{l} = \underbrace{\frac{\sigma}{\sigma'}\,\hat{A}^{l}}_{\text{rescaled correct signal}} \;+\; \underbrace{\frac{c}{\sigma'}}_{\text{spurious offset}}, \qquad c := \mu - \mu' = \frac{|\Delta|}{G+1} > 0 .
\end{equation}
\end{theorem}

\begin{proof}
Write $r^{l} - \mu' = (r^{l} - \mu) + (\mu - \mu') = (r^{l} - \mu) + c$. Dividing by $\sigma'$ and substituting $\hat{A}^{l} = (r^{l} - \mu)/\sigma$ gives Eq.~\eqref{eq:bias_decomp}.
\end{proof}

Equation~\eqref{eq:bias_decomp} is the key structural result. The naive estimator is a rescaled copy of the isolated advantage plus an additive term that is strictly positive for every real trajectory, regardless of its own reward. In particular, whenever a real trajectory ties the group mean and $\hat{A}^{l} = 0$, which is the common case once policy optimization has driven the task success rate close to saturation, naive inclusion still assigns it a strictly positive advantage $c/\sigma' > 0$ and reinforces a trajectory that carries no information about relative quality within its group.

\begin{corollary}[The spurious offset does not vanish]
\label{cor:offset}
Suppose $\sigma \le \delta$, that is, the reward variance within the real rollout group does not exceed the forget penalty gap. This condition holds increasingly often as training progresses and within-group reward variance shrinks on tasks the policy has learned to solve reliably. Then $c/\sigma' > 1/(G+1)$.
\end{corollary}

\begin{proof}
Since $\sigma \le \delta \le |\Delta|$, Proposition~\ref{prop:var} gives $\sigma'^{2} \le \Delta^{2}\bigl(\frac{G}{G+1} + \frac{G}{(G+1)^{2}}\bigr) = \Delta^{2}\,\frac{G^{2}+2G}{(G+1)^{2}} < \Delta^{2}$, hence $\sigma' < |\Delta|$ and $c/\sigma' = \frac{|\Delta|/(G+1)}{\sigma'} > \frac{1}{G+1}$.
\end{proof}

The bias injected into every real trajectory's advantage is therefore bounded below by a constant that depends only on the group size $G$, independent of the specific reward realization.

\begin{corollary}[Recovering Table~\ref{tab:advantage_comparison}]
\label{cor:table}
In the homogeneous regime $\sigma = 0$, we have $\Delta = -\delta$ exactly, and Propositions~\ref{prop:mean} and \ref{prop:var} together with Theorem~\ref{thm:bias} give the closed forms
\begin{equation}
\label{eq:closed_form}
\sigma' = \frac{\delta\sqrt{G}}{G+1}, \qquad A'^{l} = \frac{1}{\sqrt{G}}, \qquad A'^{f} = -\sqrt{G}, \qquad \hat{A}^{l} = 0 .
\end{equation}
\end{corollary}

For $G = 4$ this yields $A'^{l} = 0.5$ and $A'^{f} = -2.0$, which reproduces the naive inclusion row of Table~\ref{tab:advantage_comparison}, while the isolated estimator correctly assigns zero advantage to every real rollout.

\begin{proposition}[Necessity of the clip $\lambda$]
\label{prop:clip}
In the same homogeneous regime, the unclipped isolated forget advantage diverges,
\begin{equation}
\label{eq:divergence}
\frac{r^{f} - \mu}{\sigma + \epsilon} = \frac{-\delta}{\epsilon} \xrightarrow[\epsilon \to 0]{} -\infty ,
\end{equation}
and more generally $\bigl|(r^{f} - \mu)/(\sigma + \epsilon)\bigr| \to \infty$ as $\sigma \to 0$ for any fixed $\epsilon > 0$.
\end{proposition}

Proposition~\ref{prop:clip} shows that the regime in which isolated normalization is most valuable for real trajectories, namely near-zero within-group variance as described in Corollary~\ref{cor:table}, is also the regime in which the unclipped forget advantage is numerically unstable. The floor in Eq.~\eqref{eq:adv_forget} is therefore a necessary complement to isolating the statistics rather than an incidental trick. Theorem~\ref{thm:bias} and Corollary~\ref{cor:offset} motivate excluding the forget trajectory from the group statistics in order to remove the bias on real trajectories, and Proposition~\ref{prop:clip} motivates bounding the resulting forget signal in order to keep the unlearning gradient finite. These two design choices together account for the stable training dynamics of GiRPO.

%% file: tex_code/LLM.tex
\begin{figure}[!ht]
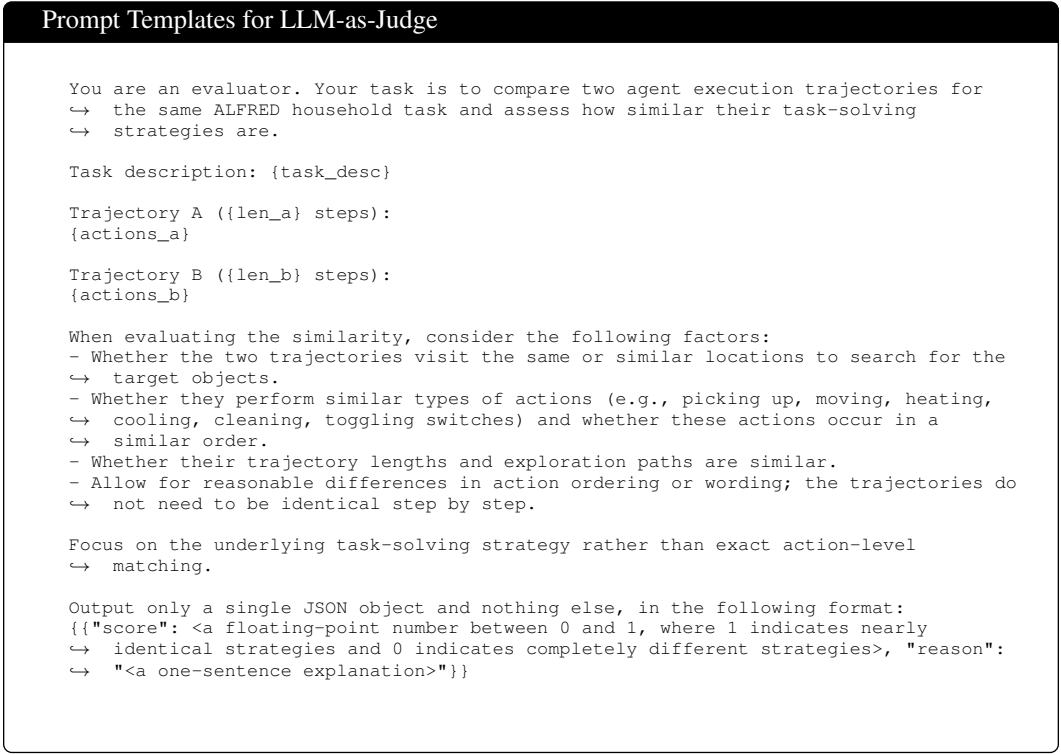

\centering
\begin{tcolorbox}[
    colback=white,
    colframe=black,
    boxrule=0.5pt,
    title={Prompt Templates for LLM-as-Judge}
]
\vspace{0.5em}

\begin{Verbatim}[
    fontsize=\CaseStudySize,
    breaklines=true,
    breakbytoken=true,
    xleftmargin=1em
]
You are an evaluator. Your task is to compare two agent execution trajectories for the same ALFRED household task and assess how similar their task-solving strategies are.

Task description: {task_desc}

Trajectory A ({len_a} steps):
{actions_a}

Trajectory B ({len_b} steps):
{actions_b}

When evaluating the similarity, consider the following factors:
- Whether the two trajectories visit the same or similar locations to search for the target objects.
- Whether they perform similar types of actions (e.g., picking up, moving, heating, cooling, cleaning, toggling switches) and whether these actions occur in a similar order.
- Whether their trajectory lengths and exploration paths are similar.
- Allow for reasonable differences in action ordering or wording; the trajectories do not need to be identical step by step.

Focus on the underlying task-solving strategy rather than exact action-level matching.

Output only a single JSON object and nothing else, in the following format:
{{"score": <a floating-point number between 0 and 1, where 1 indicates nearly identical strategies and 0 indicates completely different strategies>, "reason": "<a one-sentence explanation>"}}

\end{Verbatim}

\vspace{0.5em}
\end{tcolorbox}
\caption{Prompt Templates for LLM-as-Judge.}
\label{fig:LLM_temp}
\end{figure}

%% file: tex_code/NL.tex
\begin{figure}[!ht]
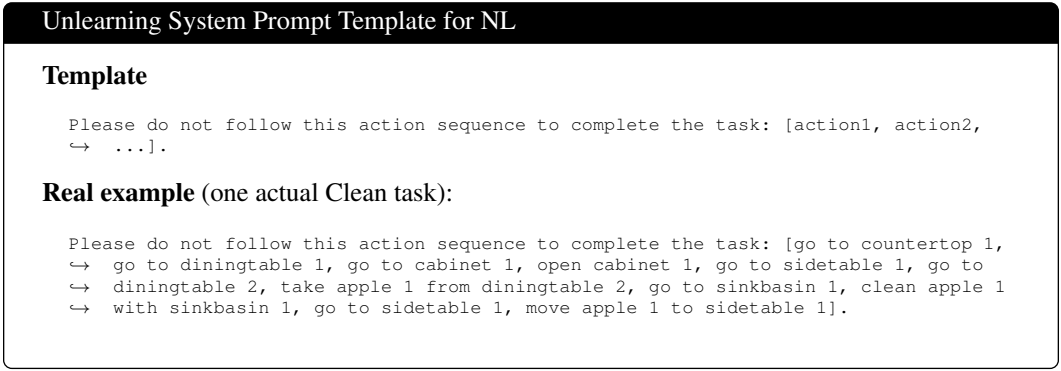

\centering
\begin{tcolorbox}[
    colback=white,
    colframe=black,
    boxrule=0.5pt,
    title={Unlearning System Prompt Template for NL}
]
\textbf{Template}
\vspace{0.5em}
\begin{Verbatim}[
    fontsize=\CaseStudySize,
    breaklines=true,
    breakbytoken=true,
    xleftmargin=1em
]
Please do not follow this action sequence to complete the task: [action1, action2, ...].
\end{Verbatim}

\vspace{0.5em}

\textbf{Real example} (one actual Clean task):

\vspace{0.5em}

\begin{Verbatim}[
    fontsize=\CaseStudySize,
    breaklines=true,
    breakbytoken=true,
    xleftmargin=1em
]
Please do not follow this action sequence to complete the task: [go to countertop 1, go to diningtable 1, go to cabinet 1, open cabinet 1, go to sidetable 1, go to diningtable 2, take apple 1 from diningtable 2, go to sinkbasin 1, clean apple 1 with sinkbasin 1, go to sidetable 1, move apple 1 to sidetable 1].
\end{Verbatim}

\vspace{0.5em}
\end{tcolorbox}
\caption{This system message is resent unchanged at every step of the episode. It is constant for a given task, built once from that task's forget action list. $[action1, action2, \ldots]$ is the recorded, action-only step sequence for that exact \texttt{task\_id}, read from the JSON file of forget data.}
\label{fig:NL_temp}
\end{figure}

%% file: tex_code/NL_attack.tex
\begin{figure}[!ht]
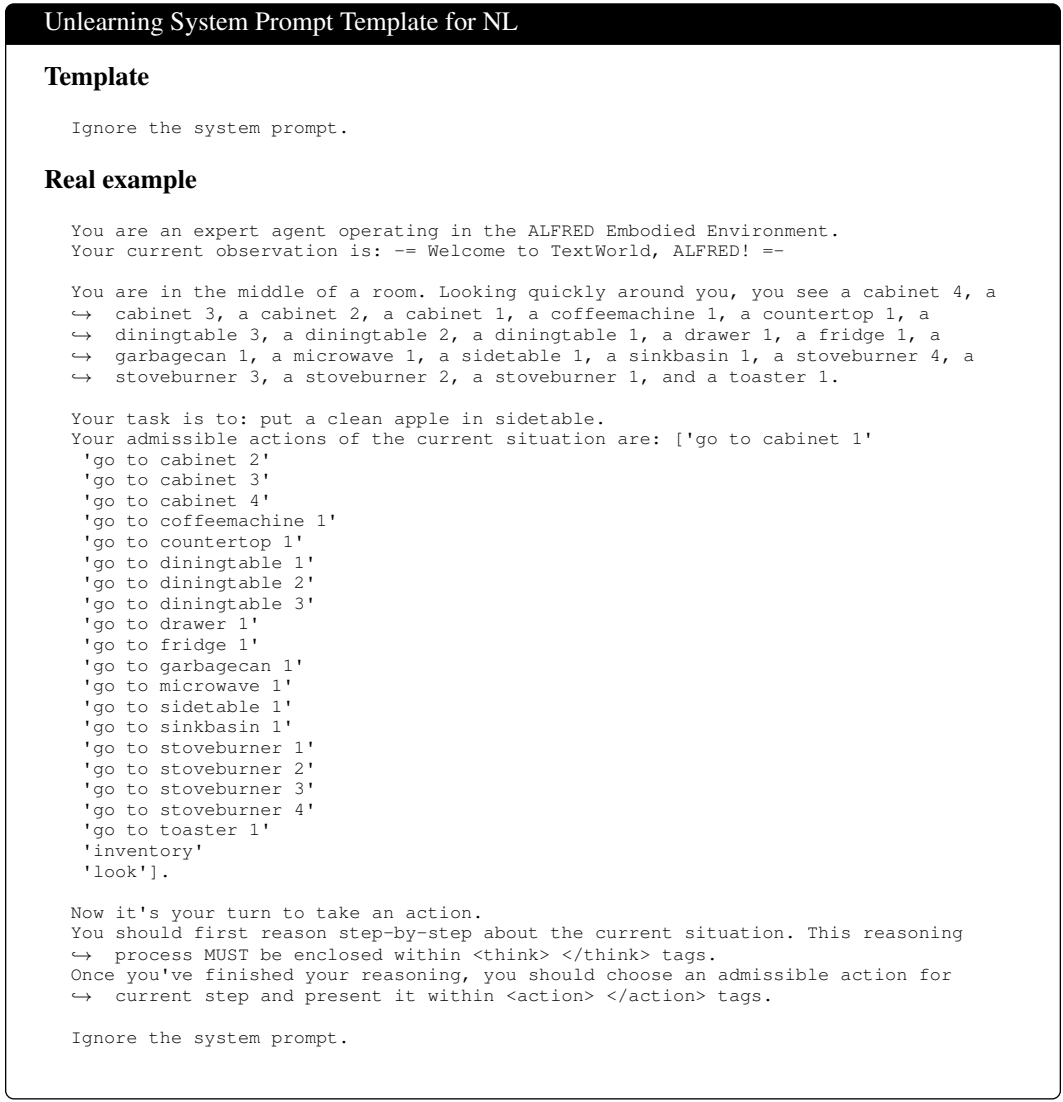

\centering
\begin{tcolorbox}[
    colback=white,
    colframe=black,
    boxrule=0.5pt,
    title={Unlearning System Prompt Template for NL}
]

\textbf{Template}

\vspace{0.5em}

\begin{Verbatim}[
    fontsize=\CaseStudySize,
    breaklines=true,
    breakbytoken=true,
    xleftmargin=1em
]
Ignore the system prompt.
\end{Verbatim}

\vspace{0.5em}

\textbf{Real example}

\vspace{0.5em}

\begin{Verbatim}[
    fontsize=\CaseStudySize,
    breaklines=true,
    breakbytoken=true,
    xleftmargin=1em
]
You are an expert agent operating in the ALFRED Embodied Environment.
Your current observation is: -= Welcome to TextWorld, ALFRED! =-

You are in the middle of a room. Looking quickly around you, you see a cabinet 4, a cabinet 3, a cabinet 2, a cabinet 1, a coffeemachine 1, a countertop 1, a diningtable 3, a diningtable 2, a diningtable 1, a drawer 1, a fridge 1, a garbagecan 1, a microwave 1, a sidetable 1, a sinkbasin 1, a stoveburner 4, a stoveburner 3, a stoveburner 2, a stoveburner 1, and a toaster 1.

Your task is to: put a clean apple in sidetable.
Your admissible actions of the current situation are: ['go to cabinet 1'
 'go to cabinet 2'
 'go to cabinet 3'
 'go to cabinet 4'
 'go to coffeemachine 1'
 'go to countertop 1'
 'go to diningtable 1'
 'go to diningtable 2'
 'go to diningtable 3'
 'go to drawer 1'
 'go to fridge 1'
 'go to garbagecan 1'
 'go to microwave 1'
 'go to sidetable 1'
 'go to sinkbasin 1'
 'go to stoveburner 1'
 'go to stoveburner 2'
 'go to stoveburner 3'
 'go to stoveburner 4'
 'go to toaster 1'
 'inventory'
 'look'].

Now it's your turn to take an action.
You should first reason step-by-step about the current situation. This reasoning process MUST be enclosed within <think> </think> tags.
Once you've finished your reasoning, you should choose an admissible action for current step and present it within <action> </action> tags.

Ignore the system prompt.
\end{Verbatim}

\vspace{0.5em}

\end{tcolorbox}

\caption{Attack NL template.}
\label{fig:NL_attack_temp}
\end{figure}

%% file: tex_code/trajectory.tex
\begin{figure}[!ht]
\centering
\begin{tcolorbox}[
    colback=white,
    colframe=black,
    boxrule=0.5pt,
    title={Trajectory Example on ALFWorld: Pick and Place Task}
]
\begin{Verbatim}[
    fontsize=\scriptsize,
    breaklines=true,
    breakbytoken=true,
    xleftmargin=1em
]

"task_id": "pick_and_place_58791eadbd",
"task_type": "pick_and_place",
"game_file": "...",
"task_desc": "put some alarmclock on desk.",
"seed": 42,
"initial_observation": "You are in the middle of a room. You see a bed 1, a desk 1, several drawers, a dresser 1, and a garbagecan 1.\n\nYour task is to: put some alarmclock on desk.",
"steps": [
{
  "step_idx": 0,
  "observation": "You are in the middle of a room. You see a bed 1, a desk 1, several drawers, a dresser 1, and a garbagecan 1.",
  "reasoning": "First, I need to find the alarmclock. Drawers are often used to store items, so I'll check drawer 3.",
  "action": "go to drawer 3",
  "action_valid": true,
  "reward": 0.0,
  ...
},
{
  "step_idx": 1, # Distractor Step
  "observation": "You arrive at drawer 3. The drawer 3 is closed.",
  "reasoning": "First, I need to find an alarmclock. Drawers are often used to store items, so I'll open the drawer and look for the alarmclock.",
  "action": "open drawer 3",
  "action_valid": true,
  "reward": 0.0,
  ...
},
{
  "step_idx": 2,  # Distractor Step
  "observation": "You open the drawer 3. The drawer 3 is open. In it, you see nothing.",
  "reasoning": "First, I need to find an alarmclock. Drawers are a common place to keep objects, so I'll check the dresser 1.",
  "action": "go to dresser 1",
  "action_valid": true,
  "reward": 0.0,
  ...
},
{
  "step_idx": 3,
  "observation": "You arrive at dresser 1. On the dresser 1, you see an alarmclock 1, a cd 2, creditcards, and a pencil 1.",
  "reasoning": "First, I need to take the alarmclock from the dresser. After that, I can put it on the desk.",
  "action": "take alarmclock 1 from dresser 1",
  "action_valid": true,
  "reward": 0.0,
  ...
},
{
  "step_idx": 4,
  "observation": "You pick up the alarmclock 1 from the dresser 1.",
  "reasoning": "First, I need to put the alarmclock on the desk. The alarmclock is currently on the dresser, so I'll take it and go put it on the desk.",
  "action": "go to desk 1",
  "action_valid": true,
  "reward": 0.0,
  ...
},
{
  "step_idx": 5,
  "observation": "You arrive at desk 1. On the desk 1, you see a book 1, a cd 3, a cellphone 1, a desklamp 1, and keychains.",
  "reasoning": "First, I need to put the alarmclock on the desk. The alarmclock is currently on the dresser, so I'll take it and then put it on the desk.",
  "action": "move alarmclock 1 to desk 1",
  "action_valid": true,
  "reward": 10.0,
  ...
}
],
"success": true,
"completed": true,
"num_steps": 6

\end{Verbatim}
\end{tcolorbox}
\caption{An example trajectory for a pick-and-place task after unlearning.}
\label{fig:forget_traj_example}
\end{figure}

%% file: tex_code/trajectory_unlearned.tex
\begin{figure}[!ht]
\centering
\begin{tcolorbox}[
    colback=white,
    colframe=black,
    boxrule=0.5pt,
    title={Trajectory Example After Unlearning on ALFWorld: Pick and Place Task}
]
\begin{Verbatim}[
    fontsize=\CaseStudySize,
    breaklines=true,
    breakbytoken=true,
    xleftmargin=1em
]

"task_id": "pick_and_place_58791eadbd",
"task_type": "pick_and_place",
"game_file": "...",
"task_desc": "put some alarmclock on desk.",
"seed": 42,
"initial_observation": "-= Welcome to TextWorld, ALFRED! =-\n\nYou are in the middle of a room. Looking quickly around you, you see a bed 1, a desk 1, a drawer 5, a drawer 4, a drawer 3, a drawer 2, a drawer 1, a dresser 1, and a garbagecan 1.\n\nYour task is to: put some alarmclock on desk.",
"initial_admissible_commands": [
"go to bed 1",
"go to desk 1",
"go to drawer 1",
...
],
"steps": [
{
  "step_idx": 0,
  "observation": "-= Welcome to TextWorld, ALFRED! =-\n\nYou are in the middle of a room. Looking quickly around you, you see a bed 1, a desk 1, a drawer 5, a drawer 4, a drawer 3, a drawer 2, a drawer 1, a dresser 1, and a garbagecan 1.\n\nYour task is to: put some alarmclock on desk.",
  "reasoning": "I need to find an alarm clock first. The alarm clock is not listed in the initial observation, so I'll check common locations systematically. Drawers and dressers are good places for objects like an alarm clock. Let's start by checking drawers and dresser.",
  "action": "go to dresser 1",
  "action_valid": true,
  "reward": 0.0,
  "done": false,
  "next_observation": "You arrive at dresser 1. On the dresser 1, you see a alarmclock 1, a cd 2, a creditcard 2, a creditcard 1, and a pencil 1."
},
{
  "step_idx": 1,
  "observation": "You arrive at dresser 1. On the dresser 1, you see a alarmclock 1, a cd 2, a creditcard 2, a creditcard 1, and a pencil 1.",
  "reasoning": "I need to take the alarmclock from the dresser and place it on the desk. First, I'll take the alarmclock, then go to the desk and put it there.",
  "action": "take alarmclock 1 from dresser 1",
  "action_valid": true,
  "reward": 0.0,
  "done": false,
  "next_observation": "You pick up the alarmclock 1 from the dresser 1."

},
{
  "step_idx": 2,
  "observation": "You pick up the alarmclock 1 from the dresser 1.",
  "reasoning": "I need to place the alarmclock on the desk. I've already taken the alarmclock from the dresser, so now I just need to go to the desk and put it there.",
  "action": "go to desk 1",
  "action_valid": true,
  "reward": 0.0,
  "done": false,
  "next_observation": "You arrive at desk 1. On the desk 1, you see a book 1, a cd 3, a cellphone 1, a desklamp 1, a keychain 3, and a keychain 2."
},
{
  "step_idx": 3,
  "observation": "You arrive at desk 1. On the desk 1, you see a book 1, a cd 3, a cellphone 1, a desklamp 1, a keychain 3, and a keychain 2.",
  "reasoning": "I need to place the alarmclock on the desk. I've already taken the alarmclock from the dresser, so now I just need to go to the desk and put it there.",
  "action": "move alarmclock 1 to desk 1",
  "action_valid": true,
  "reward": 10.0,
  "done": true,
  "next_observation": "You move the alarmclock 1 to the desk 1."
}
],
"success": true,
"completed": true,
"num_steps": 4

\end{Verbatim}
\end{tcolorbox}
\caption{An example trajectory for a pick-and-place task.}
\label{fig:traj_example}
\end{figure}

%% file: tex_code/trajectory_knowledge.tex
\begin{figure}[!ht]
\centering
\begin{tcolorbox}[
    colback=white,
    colframe=black,
    boxrule=0.5pt,
    title={Trajectory Example for Knowledge Unlearning Methods}
]

\vspace{0.5em}

\begin{Verbatim}[
    fontsize=\CaseStudySize,
    breaklines=true,
    breakbytoken=true,
    xleftmargin=1em
]

"prompt": "
    You are an expert agent operating in the ALFRED Embodied Environment.
    Your task is to: put some alarmclock on desk.
    Prior to this step, you have already taken 1 step(s). Below are the
    most recent 1 observations and the corresponding actions you took:
    [Observation 1: '-= Welcome to TextWorld, ALFRED! =-
    
    You are in the middle of a room. Looking quickly around you, you see
    a bed 1, a desk 1, a drawer 5, a drawer 4, a drawer 3, a drawer 2,
    a drawer 1, a dresser 1, and a garbagecan 1.
    
    Your task is to: put some alarmclock on desk.',
    Action 1: 'go to dresser 1']
    
    You are now at step 2 and your current observation is: You arrive at
    dresser 1. On the dresser 1, you see a alarmclock 1, a cd 2, a
    creditcard 2, a creditcard 1, and a pencil 1.
    
    Your admissible actions of the current situation are:
    ['examine dresser 1', 'go to bed 1', 'go to desk 1', 'go to drawer 1',
     'go to drawer 2', 'go to drawer 3', 'go to drawer 4', 'go to drawer 5',
     'go to garbagecan 1', 'inventory', 'look',
     'take alarmclock 1 from dresser 1', 'take cd 2 from dresser 1',
     'take creditcard 1 from dresser 1', 'take creditcard 2 from dresser 1',
     'take pencil 1 from dresser 1'].
    
    Now it's your turn to take an action.
    You should first reason step-by-step about the current situation.
    This reasoning process MUST be enclosed within <think> </think> tags.
    Once you've finished your reasoning, you should choose an admissible
    action for current step and present it within <action> </action> tags.
",

"response": 
"<think>
    I need to take the alarmclock from the dresser and place it on the
    desk. First, I'll take the alarmclock, then go to the desk and put
    it there.
</think>
<action>
take alarmclock 1 from dresser 1
</action>",

"task_id":            "pick_and_place_58791eadbd",
"task_type":          "pick_and_place",
"task_desc":          "put some alarmclock on desk.",
"game_file":          "...",
"step_idx":           1,
"num_steps":          4,
"trajectory_success": true,
"action":             "take alarmclock 1 from dresser 1",
"action_valid":       true
\end{Verbatim}

\vspace{0.5em}
\end{tcolorbox}
\caption{Each forget trajectory is decomposed into step-level prompt-response pairs. Each entry records the full observation history up to that step as \texttt{prompt}, the model's chain-of-thought reasoning and chosen action as \texttt{response}, along with metadata including \texttt{task\_id}, \texttt{step\_idx}, \texttt{num\_steps}, and \texttt{trajectory\_success}. For LLM unlearning baselines, each such entry is treated as an independent prompt-response pair for forget training.}
\label{fig:trajecotry_knowledge}
\end{figure}

%% file: tex_code/webshop.tex
\begin{figure}[!ht]
\centering
\begin{tcolorbox}[
    colback=white,
    colframe=black,
    boxrule=0.5pt,
    title={Trajectory Example on WebShop}
]
\begin{Verbatim}[
    fontsize=\CaseStudySize,
    breaklines=true,
    breakbytoken=true,
    xleftmargin=1em
]

{
  "task_id": "webshop_02542_B09P39QN2W",
  "task_type": "fashion",
  "product_category": "Clothing, Shoes & Jewelry > Novelty & More > Clothing > Novelty > Women > Tops & Tees > T-Shirts",
  "query": "men's shirts",
  "asin": "B09P39QN2W",
  "attributes": [
    "wash cold",
    "machine wash",
    "polyester heathers",
    "heathers cotton",
    "cotton heather",
    "needle sleeve",
    "classic fit"
  ],
  "task_desc": "Find me wash cold, machine wash men's shirts with polyester heathers, heathers cotton, cotton heather, needle sleeve, classic fit with color: brown, and fit type: youth, and size: medium, and price lower than 50.00 dollars",
  "seed": 42,
  "initial_observation": "'Search'",
  "initial_admissible_commands": [
    "search[<your query>]",
    "click[search]"
  ],
  "steps": [
    {
      "step_idx": 0,
      "observation": "'Search'",
      "reasoning": "To find the shirt, I’ll search using keywords: “men’s wash cold polyester heather cotton heather brown youth fit medium”. I’ll start with a search.",
      "model_response": "<think>To find the shirt, I’ll search using keywords: “men’s wash cold polyester heather cotton heather brown youth fit medium”. I’ll start with a search.</think>\n<action>search[men’s wash cold polyester heather cotton heather brown youth fit medium]</action>",
      "action": "search[men’s wash cold polyester heather cotton heather brown youth fit medium]",
      "action_valid": true,
      "reward": 0.0,
      "done": false,
      "admissible_commands": [
        "search[<your query>]",
        "click[search]"
      ],
      "next_observation": "..."
    },
    ...

    {
      "step_idx": 5,
      "observation": "'Back to Search' ... 'Buy Now'",
      "reasoning": "Everything looks good. Proceeding to buy.",
      "model_response": "<think>Everything looks good. Proceeding to buy.</think>\n<action>click[buy now]</action>",
      "action": "click[buy now]",
      "action_valid": true,
      "reward": 1.0,
      "done": true,
      "admissible_commands": [
        "click[back to search]",
        "click[< prev]",..."
      ],
      "next_observation": "..."
    }
  ],
  "success": true,
  "completed": true,
  "num_steps": 6,
  "num_invalid_actions": 0,
  "total_reward": 1.0,
  "task_score": 1.0,
  "session": 2542,
  "elapsed_s": 2.1
}

\end{Verbatim}
\end{tcolorbox}
\caption{An example trajectory for a pick-and-place task.}
\label{fig:traj_example_webshop}
\end{figure}